%% file: main.tex
\documentclass[letterpaper]{article} %

\input{packages}

\input{macros}

\input{titlepage}

\begin{document}

\maketitle

\input{Sections/abstract}

\input{Sections/introduction}
\input{Sections/prelims}
\input{Sections/algorithm}
\input{Sections/empirical}
\input{Sections/conclusion}

\bibliography{refs}

\clearpage
\appendix

\input{Sections/appendix}

\end{document}

%% file: packages.tex
\usepackage{aaai2027}
\nocopyright

\usepackage[hyphens]{url}  %
\usepackage{graphicx}      %
\usepackage{natbib}        %
\usepackage{caption}       %
\usepackage{algorithm}
\usepackage{algorithmic}

\usepackage{newfloat}
\usepackage{listings}
\DeclareCaptionStyle{ruled}{labelfont=normalfont,labelsep=colon,strut=off} %
\floatstyle{ruled}
\newfloat{listing}{tb}{lst}{}
\floatname{listing}{Listing}

\usepackage{booktabs}

\usepackage[T1]{fontenc}
\usepackage[utf8]{inputenc}
\usepackage{textcomp}

\usepackage{amsmath,amssymb,amsthm}
\usepackage{mathtools}

\usepackage[dvipsnames]{xcolor}
\usepackage{tikz}
\usetikzlibrary{automata, arrows.meta, positioning}

\usepackage{multirow}
\usepackage{enumitem}
\usepackage{comment}

%% file: macros.tex
\newcommand{\E}{\mathbb{E}}                 %
\newcommand{\N}{\mathbb{N}}                 %

\DeclareMathOperator*{\argmax}{arg\,max}

\renewcommand{\S}{\mathcal{S}}   %
\newcommand{\A}{\mathcal{A}}     %
\renewcommand{\P}{P}             %
\newcommand{\snot}{i_0}          %

\renewcommand{\Pr}[1]{\P(#1 \mid s, \, a)}

\newcommand{\SApair}[1]{s_{#1}, \, a_{#1}}

\newcommand{\Snt}{\S_{\text{nt}}}
\newcommand{\Sterm}{\S_{\text{term}}}
\newcommand{\ST}{\S_\mathcal{G}}
\newcommand{\SN}{\S_{{\neg \mathcal{G}}}}

\newcommand{\SAnt}{\Snt \times \A}              %
\newcommand{\maxSAnt}{\max_{s, a \in \SAnt}}    %

\newcommand{\Vstar}{V^*}
\newcommand{\Qstar}{Q^*}                        %
\newcommand{\OptimalP}{\pi^*}
\renewcommand{\H}{\mathcal{H}}   %

\newcommand{\wnorm}[1]{\left\| #1 \right\|_{w}}      %
\newcommand{\err}{n}                                 %

\newcommand{\QUASAR}{\textsc{Quasar}}                %
\newcommand{\QUASARN}{\textsc{Quasar-Naive}}         %
\newcommand{\QUASARBR}{\textsc{Quasar-BR}}           %
\newcommand{\SPAC}{\textsc{Staged-PAC}}              %
\newcommand{\verr}{v\text{-}\mathrm{err}_{\snot}}    %
\newcommand{\qerr}{q\text{-}\mathrm{err}_{\max}}     %

\newcommand{\Theoremspace}{\vspace{0.15cm}}   %

\newtheorem{theorem}{Theorem}
\newtheorem{corollary}[theorem]{Corollary}
\newtheorem{lemma}{Lemma}

\newtheorem{definition}{Definition}
\newtheorem{assumption}{Assumption}
\newtheorem{property}{Property}

%% file: titlepage.tex
\title{Q-Learning for Reachability in MEC-Free MDPs}
\author{
    Lu-Chin Chang,
    Suguman Bansal
}
\affiliations{
    Georgia Institute of Technology\\
    luchin\_chang@gatech.edu, suguman@gatech.edu
}

%% file: Sections/abstract.tex
\begin{abstract}
{\em Reinforcement learning} (RL) for {\em reachability specifications} is
fundamental to sequential decision-making. Prior work establishes
{\em asymptotic convergence} to optimal policies, but only through
{\em model-based} methods that must explicitly estimate the transition
probabilities of the underlying Markov Decision Process (MDP). We present
\QUASAR{}, the first {\em model-free} algorithm with asymptotic guarantees for
reachability on the fragment of MDPs free of non-terminal {\em maximal end
components (MECs)}, a building block to which every MDP reduces by the standard
MEC quotient. Our algorithm follows the classical {\em Q-learning}
approach, using temporal-difference updates to converge to an optimal
policy without ever learning the transition probabilities. The resulting learner
reduces the memory footprint from the $O(|S|^2 |A|)$ that
model-based methods require to $O(|S| |A|)$. On the
standardized Quantitative Verification Benchmark Set, our algorithm converges to the optimal
policy with orders of magnitude fewer samples than the previous model-based
state-of-the-art. Together these
results are a concrete step toward the practical deployment of reachability
learning and, with it, of specification-guided RL.
\end{abstract}

%% file: Sections/introduction.tex
\section{Introduction}
\label{sec:intro}

Reinforcement learning (RL)~\cite{suttonbarto2018rl} is the standard framework for sequential decision-making in unknown Markov decision processes (MDPs), and cast as reward maximization it admits strong theoretical guarantees. The strongest is \emph{asymptotic convergence}~\cite{watkins1992qlearning}, under which the learned policies are optimal in the limit almost surely, and it is on such guarantees that the practical success of RL rests, from robotic locomotion~\cite{collins2005bipedal} and dexterous manipulation~\cite{andrychowicz2020dexterous} to visuomotor control~\cite{levine2016end}.

In this work we seek to carry this guarantee into the setting of \emph{formal specifications}, where an agent must satisfy a temporal objective rather than maximize a hand-tuned reward~\cite{alur2026specification}. Specifications in linear temporal logic (LTL) and related $\omega$-regular formalisms express rich behaviours, such as safety and liveness, that are awkward to encode as rewards~\cite{hahn2019omega}. It is worth noting that an asymptotic guarantee is close to the best one can hope for here, since the finite-sample (PAC) guarantees available for reward-based RL largely fail for specifications. In particular, general LTL objectives are not PAC-learnable~\cite{yang2021intractability,alur2021framework}, and the PAC guarantees for restricted classes rest on an unobservable quantity of the MDP, such as the smallest transition probability~\cite{ashok2019pac,perez2023pac,svoboda2024reachability}.

Among all specifications, \emph{reachability} is the most fundamental, the primitive from which the richer LTL and $\omega$-regular objectives are built, so that an algorithm learning it with a given guarantee immediately yields one for that entire class (Section~\ref{sec:related}). This makes reachability the linchpin of specification-guided RL, and it is the problem we take up.

For reachability, asymptotic guarantees have been established only recently, and only by \emph{model-based} means~\cite{palasamudram2026asymptotic,le2024reinforcement}. Building a model is a genuine hindrance at scale. A model-based learner stores a transition estimate for every state-action-successor triple, at $O(|\S|^2|\A|)$ memory, and must refine it to uniform accuracy before planning, so its sample and time costs track the model size rather than the difficulty of the objective. On the large product MDPs that specifications induce, both are prohibitive. One naturally wonders whether reachability admits a \emph{model-free} learner that still converges asymptotically.

We answer this affirmatively. We present \QUASAR{}, the first model-free algorithm with asymptotic guarantees for reachability. We operate on the \emph{no-EC} fragment of MDPs, those with no maximal end component (MEC) among the non-terminal states, a standard structural condition that prevents the agent from postponing termination forever. Every MDP reduces to the no-EC fragment by the standard MEC quotient~\cite{deAlfaro1997}, so the fragment is a building block for the general MDP setting rather than a restriction of it. \QUASAR{} is deliberately in the style of Q-learning~\cite{watkins1992qlearning}. It keeps a single $O(|\S|\,|\A|)$ table of value estimates, refines them from sampled transitions, and never forms a model. The difficulty lies not in the algorithm but in its analysis. Classical Q-learning converges because a discount factor $\gamma \in (0,1)$ makes the Bellman operator a contraction, and standard discounted-sum RL rewards a trajectory that first reaches the target at step $k$ with $\gamma^k$. Reachability is precisely the undiscounted case $\gamma = 1$ of this scheme, where the contraction that drives Q-learning to converge disappears, so the classical proof does not carry over. The natural remedy, recasting reachability as some other discounted-sum reward, is ruled out by Alur et al.~\cite{alur2021framework}. No optimality-preserving reduction from reachability to discounted-sum rewards exists, and the gap must therefore be closed directly.

We close it by replacing the missing discount factor with a quantity intrinsic to the MDP, the expected time to termination. Measuring convergence in a norm weighted by this hitting time, we show that the undiscounted Bellman operator contracts after all, with the expected hitting time playing exactly the role that the discount factor plays in the classical analysis. This weight is a property of the MDP, not a parameter the learner must know or tune, so our guarantee rests on no numerical quantity of the MDP, only on the structural no-EC assumption. A stochastic-approximation argument then lifts the contraction from the exact operator to the sampled updates, yielding convergence with probability one. The hitting-time weight may recall the stochastic shortest path problem, but the two are distinct and neither reduces to the other, so its contraction machinery does not apply off the shelf. We return to this in Section~\ref{sec:related}.

These guarantees translate into practice. We evaluate two versions of \QUASAR{}, the online \QUASARN{} and a backward-replay variant \QUASARBR{} that preserves the guarantee (Theorem~\ref{thm:br}). On the standardized Quantitative Verification Benchmark Set~\cite{qvbs2019}, \QUASARBR{} converges to the optimal value on every benchmark, using orders of magnitude fewer samples than the model-based state of the art.

The key contribution of this work is to establish the following two results, formalized as Theorem~\ref{thm:main} and Theorem~\ref{thm:space}.
\begin{enumerate}
    \item On no-EC MDPs, the iterates of \QUASAR{} converge to the optimal reachability values almost surely (Theorem~\ref{thm:main}). It is worth emphasizing that the sole hypothesis is structural (Assumption~\ref{asm:finite_hitting}), and the guarantee is conditional on no numerical parameter of the MDP.
    \item \QUASAR{} runs in $O(|\S|\,|\A|)$ space (Theorem~\ref{thm:space}), a quadratic improvement over the $O(|\S|^2 |\A|)$ that model-based methods spend on stored transition estimates.
\end{enumerate}
To the best of our knowledge, \QUASAR{} is the first algorithm to attain an asymptotic guarantee for reachability without either estimating a model or reducing the objective to rewards.

%% file: Sections/prelims.tex
\section{Preliminaries and Problem Setup}
\label{sec:prelims}

\subsection{Preliminaries}

\paragraph{Markov Decision Processes (MDPs)}
A finite \emph{Markov Decision Process} (MDP) is a tuple $(\S, \, \A, \, \{\A(s)\}_{s \in \S}, \, \snot, \, \P)$ consisting of a finite set of \emph{states} $\S$, a finite set of \emph{actions} $\A$ with a non-empty set of \emph{enabled actions} $\A(s) \subseteq \A$ at each state $s$, an \emph{initial state} $\snot \in \S$, and a \emph{transition kernel} $\P$ that assigns to every $s$ and enabled $a \in \A(s)$ a distribution $\P(\cdot \mid \SApair{}) \in \Delta(\S)$ over successors. Here $\Pr{s'}$ is the probability of moving from $s$ to $s'$ under $a$. We partition $\S$ into \emph{non-terminal states} $\Snt$ and \emph{terminal states} $\Sterm$, the latter absorbing ($\Pr{s} = 1$ for $s \in \Sterm$, $a \in \A(s)$) and further split into \emph{target states} $\ST$ and \emph{non-target states} $\SN = \Sterm \setminus \ST$. We assume without loss of generality that every state is reachable from the initial state $\snot$, since any state unreachable from $\snot$ can be discarded without affecting the reachability values.

\paragraph{Policies, Trajectories, and Reachability}

A \emph{deterministic policy} $\pi : \S \to \A$ with $\pi(s) \in \A(s)$ fixes an action at each state, inducing a stochastic process $\{(s_t, a_t)\}_{t \geq 0}$ with $s_0 = \snot$, $a_t = \pi(s_t)$, and $s_{t+1} \sim \P(\cdot \mid \SApair{t})$. We write $\tau = (s_0, a_0, s_1, \dots)$ for a trajectory and $\P^\pi$ for its distribution, under which $\P^\pi\bigl((s_0, \dots, s_t)\bigr) = \prod_{k=0}^{t-1} \P(s_{k+1} \mid \SApair{k})$.

For a set of states $S \subseteq \S$, we write $\tau \models S$ when $s_n \in S$ for some $n \in \N$, i.e., the trajectory visits $S$. The \emph{state} and \emph{state-action reachability values} of a policy $\pi$ are the probabilities of reaching the target set $\ST$ from $s$ and from $(\SApair{})$, respectively.
\begin{align}
    V^\pi(s) &\coloneqq \P^\pi(\tau \models \ST \mid s_0 = s), \label{def:vval}\\
    Q^\pi(\SApair{}) &\coloneqq \P^\pi(\tau \models \ST \mid s_0 = s,\, a_0 = a). \label{def:qval}
\end{align}
Since terminal states are absorbing, $V^\pi$ and $Q^\pi$ equal $1$ on $\ST$ and $0$ on $\SN$ for every $\pi$. 

A policy $\OptimalP$ is \emph{optimal} if $V^{\OptimalP} \geq V^\pi$ for all $\pi$. For finite MDPs a deterministic optimal policy always exists for reachability objectives, and the optimal value $\Vstar$ is unique even when several optimal policies attain it~\cite{deAlfaro1997}.

\paragraph{Reinforcement Learning}
In \emph{reinforcement learning} (RL)~\cite{suttonbarto2018rl} the dynamics of the MDP are unknown. The learner knows $\S$, $\A$, $\{\A(s)\}_{s \in \S}$, and $\snot$, but not $\P$. In its place it queries a \emph{simulator} that, given $(\SApair{})$, returns a successor $s' \sim \P(\cdot \mid \SApair{})$ from the true dynamics. A \emph{learning algorithm} interacts with the simulator (at each step sampling an action from the current state or resetting to $\snot$) and maintains an estimate of the optimal values, thereby inducing a random sequence of estimates $\{Q_t\}_{t \ge 0}$ and policies $\{\pi_t\}_{t \ge 0}$. The strongest classical guarantee is of \emph{asymptotic convergence}~\cite{watkins1992qlearning} which states that the induced policies are optimal in the limit, almost surely. Formally,
\Theoremspace
\begin{definition}[Asymptotic Guarantee]\label{def:asymptotic}
A learning algorithm \emph{converges asymptotically} if the policies $\{\pi_t\}_{t \ge 0}$ it induces satisfy $V^{\pi_t}(\snot) \to \Vstar(\snot)$ as $t \to \infty$ with probability $1$.
\end{definition}

\subsection{Problem Setup}\label{sec:setup}
We aim to learn reachability objectives with \emph{asymptotic guarantees} (Definition~\ref{def:asymptotic}) in a \emph{model-free} way, i.e., from simulator access alone without building an explicit model of $\P$. Prior algorithms with asymptotic guarantees for reachability are \emph{model-based}~\cite{palasamudram2026asymptotic,le2024reinforcement}, explicitly estimating $\P$ and then computing the optimal policy by classical optimization~\cite{deAlfaro1997}.

In this work we restrict our attention to the \emph{no-EC} fragment of MDPs, which we define through end components.
\Theoremspace
\begin{definition}[End Component and MEC]\label{def:ec}
An \emph{end component} (EC) is a non-empty set $C \subseteq \S \times \A$ that is (i)~\emph{closed}, i.e., if $(\SApair{}) \in C$ and $\P(s' \mid \SApair{}) > 0$, then $(s', a') \in C$ for some $a' \in \A(s')$, and (ii)~\emph{strongly connected}, i.e., the subgraph it induces links any two of its states via transitions within $C$. A \emph{maximal end component} (MEC) is an EC contained in no larger EC.
\end{definition}

Intuitively, an MEC is a region in which the agent can remain forever. We exclude such regions from among the non-terminal states.
Every MDP reduces to a no-EC instance by collapsing each maximal end component to a single state, the standard MEC quotient~\cite{deAlfaro1997}, so the no-EC fragment is a building block for the general MDP setting rather than a restriction of it. 

\Theoremspace
\begin{assumption}[No-EC]\label{asm:finite_hitting}
No end component contains a pair $(\SApair{})$ with $s \in \Snt$. Equivalently, no MEC contains a non-terminal state.
\end{assumption}

This frames the central question of the paper. \emph{Can reachability objectives on no-EC MDPs be learnt with asymptotic guarantees in a model-free manner?} We answer this affirmatively and constructively. Section~\ref{sec:algo} presents the algorithm, and Section~\ref{sec:convergence} establishes its guarantee (Theorem~\ref{thm:main}).

\subsection{Related Work}
\label{sec:related}

Every $\omega$-regular specification~\cite{thomas1990automata}, and hence every LTL specification~\cite{pnueli1977}, reduces to optimal reachability, a reduction standard for both known~\cite{baier2008principles} and unknown~\cite{sickert2016ldba,hahn2020gfm,hahn2020faithful} MDPs and simulable for reinforcement learning~\cite{hahn2020gfm}. It passes through a canonical automaton form~\cite{safra1988} and a probability-independent graph analysis~\cite{courcoubetis1995,chatterjee2011mec,chatterjee2014mec}, and finite-horizon ($\text{LTL}_f$) goals reduce likewise~\cite{degiacomo2013ltlf}. Any advance in learning reachability therefore transfers to the full class, placing reachability at the centre of specification-guided learning.

The only prior algorithms with asymptotic guarantees, Palasamudram et al.~\cite{palasamudram2026asymptotic} for reachability and Le et al.~\cite{le2024reinforcement} for full LTL, both estimate a model of the dynamics and reduce the objective to average reward. \QUASAR{} is the first model-free alternative, matching their guarantee with a single $O(|\S|\,|\A|)$ table in place of the $O(|\S|^2 |\A|)$ a model demands.

A separate line of model-free algorithms targets $\omega$-regular objectives directly~\cite{hahn2019omega,hahn2020gfm,hahn2020faithful}, translating the specification into a reward on an enlarged state space. That translation is faithful only past a threshold that itself depends on the transition probabilities the learner is trying to avoid estimating, so no fixed choice can be certified in advance. \QUASAR{} asks for no such parameter, resting only on the no-EC assumption (Assumption~\ref{asm:finite_hitting}).

Our hitting-time weight may recall the stochastic shortest path (SSP) problem, whose analysis also turns on a contraction and an asynchronous stochastic-approximation lemma~\cite{bertsekas1991analysis,tsitsiklis1994asynchronous}, yet the resemblance is confined to the machinery. SSP seeks a proper policy minimizing an additive path cost to a target reached with certainty, whereas reachability maximizes a target-reaching probability that need not equal one and ignores path length. Since reachability grants no proper policy in general, the SSP contraction does not transfer, and neither problem reduces to the other (Sup.\ mat.~\ref{app:ssp}).

%% file: Sections/algorithm.tex
\section{Algorithm}
\label{sec:algo}
We present our learning algorithm and explain why it computes reachability values. The proof is deferred to Section~\ref{sec:convergence}. We recall the Bellman equations, state the algorithm, and walk through its updates.
Our algorithm adapts \emph{Q-learning}~\cite{watkins1992qlearning} to reachability. Classical Q-learning cannot be reused off the shelf, as reachability admits no optimality-preserving reduction to discounted-sum rewards~\cite{alur2021framework}. We therefore learn the reachability values $\Qstar$ \emph{directly}, with no discount factor, recovering the state values as $\Vstar(s) = \max_{a \in \A(s)} \Qstar(\SApair{})$. A salient strength of the algorithm is its simplicity. It is Q-learning with the discount removed and the terminal values clamped, yet it converges on the entire no-EC fragment.
\paragraph{Bellman equations.}
The reachability values obey Bellman recursions, in which a state's value is the average of its successors' values, anchored at $1$ on the targets and $0$ on the other terminals. Formally, for a fixed policy $\pi$, non-terminal $s \in \Snt$, and $a \in \A(s)$,
\begin{align*}
    V^{\pi}(s) &= \sum_{s' \in \S} \P(s' \mid s, \, \pi(s))\, V^\pi(s'), \\
    Q^{\pi}(\SApair{}) &= \sum_{s' \in \S} \P(s' \mid s, \, a)\, V^\pi(s').
\end{align*}
Maximizing over policies and invoking the Markov property yields the \emph{Bellman optimality equations}
\begin{align}
    V^*(s) &= \max_{a \in \A(s)} \sum_{s' \in \S} \P(s' \mid \SApair{})\, V^*(s'), \label{def:vboe}\\
    Q^*(\SApair{}) &= \sum_{s' \in \S} \P(s' \mid \SApair{})\, V^*(s'), \label{def:qboe}
\end{align}
and our algorithm learns $\Qstar$ as the fixed point of~\eqref{def:qboe} from sampled transitions.
\begin{algorithm}[t]
\caption{\QUASARN{}}
\label{alg:qlearning}
\begin{algorithmic}[1]
  \STATE \textbf{Algorithm parameters:} step-size schedule $\{\alpha_n(\SApair{})\}_{n \ge 1} \subseteq (0,1]$ indexed by the visit count, satisfying Assumption~\ref{asm:lr}; exploration rate $\epsilon > 0$; convergence threshold $\eta > 0$
  \STATE Initialize $Q(s,a)=0$ and $N(s,a)=0$ for $s \in \Snt$
  \STATE Set $Q(s,a)=1$ for $s\in\ST$ and $Q(s,a)=0$ for $s\in\SN$ (fixed, never updated)
  \STATE Initialize $\Delta = \infty$
  \WHILE{$\Delta \ge \eta$}
    \STATE $\Delta \gets 0$
    \STATE Initialize starting state $s = \snot$
    \REPEAT
      \STATE Choose $a$ from $s$ using $\epsilon$-greedy policy based on $Q$
      \STATE Get next state $s' \sim P(\cdot \mid s, a)$ via simulator
      \STATE $q_{\text{old}} \gets Q(s,a)$; \; $N(s,a) \gets N(s,a) + 1$
      \STATE $Q(s,a) \gets Q(s,a)$ \\
             \hspace{1.5em}$+\, \alpha_{N(s,a)}(\SApair{})\!\left[\max_{a' \in \A(s')} Q(s'\!,a') - Q(s,a)\right]$
      \STATE $\Delta \gets \max(\Delta, |Q(s,a) - q_{\text{old}}|)$
      \STATE $s \gets s'$
    \UNTIL{$s \in \Sterm$}
  \ENDWHILE
\end{algorithmic}
\end{algorithm}
\paragraph{The learning algorithm.}
The algorithm maintains an estimate $Q$ of $\Qstar$, moving $Q(s,a)$ at each visited pair toward the sampled successor's value as~\eqref{def:qboe} dictates. Since the successor is drawn from the true dynamics, this target is an unbiased estimate of~\eqref{def:qboe}, and the update never consults $\P$, so the learner is model-free. The terminal values ($1$ at targets, $0$ elsewhere) propagate backward, driving $Q$ to the true reachability probabilities.
Concretely (Lines~1--3), $Q$ starts at $0$ on non-terminal pairs and at the boundary values on terminals. Each episode starts at $\snot$ and, at state $s$, selects an action $\epsilon$-greedily (Line~9), draws $s' \sim P(\cdot \mid \SApair{})$ (Line~10), and moves $Q(s,a)$ toward $\max_{a' \in \A(s')} Q(s',a')$ (Line~12); it ends on reaching a terminal state (Line~15), almost surely under Assumption~\ref{asm:finite_hitting}, and the outer loop repeats until the largest sweep update falls below $\eta$. No discount factor appears, and the terminal values alone carry the reward signal. By Theorem~\ref{thm:main}, $Q \to \Qstar$ almost surely, and Corollary~\ref{cor:asymptotic} converts this into the asymptotic guarantee of Definition~\ref{def:asymptotic}.
\paragraph{Space complexity.}
Learning $\Qstar$ directly, rather than through a model of $\P$, keeps the memory footprint small.
\Theoremspace
\begin{theorem}[Space complexity]\label{thm:space}
    Algorithm~\ref{alg:qlearning} runs in $O(|\S|\,|\A|)$ space.
\end{theorem}
The algorithm keeps one scalar per state-action pair, one visit counter $N(\SApair{})$ per pair (which is what makes the per-pair step sizes of Assumption~\ref{asm:lr} realizable), and a constant number of auxiliary scalars ($\Delta$, $q_{\text{old}}$, and the current state and action). Both tables are indexed by $\SAnt$, so each contributes $|\S|\,|\A|$ entries. No $|\S|^2$ term arises, since it stores no transition estimates. This is a quadratic saving over the model-based route, which stores $O(|\S|^2|\A|)$ transition entries~\cite{palasamudram2026asymptotic}, and by Theorem~\ref{thm:main} we obtain it without weakening the guarantee.
\paragraph{Backward replay.}
Run online, Algorithm~\ref{alg:qlearning} is the version we call \QUASARN{} and analyze below. A second version, \QUASARBR{}, buffers a whole episode and then applies the very same updates in reverse order, from the terminal state back to the start~\cite{lin1992}. The reversal changes neither which updates are made nor their sizes, only the order in which they land, so one expects it to reach the same values. The expectation is not automatic, since a reversed update reads entries that were rewritten later in the same episode, and we therefore prove separately that \QUASARBR{} attains the guarantee of Theorem~\ref{thm:main} (Theorem~\ref{thm:br}, Sup.\ mat.~\ref{app:br}).

Reversal earns this extra proof by making each episode count for far more. With the per-pair step size $\alpha = 1/N$, a pair's first update takes $\alpha = 1$ and simply overwrites its estimate with the successor's value, so a single backward sweep carries the terminal reward from a target all the way to $\snot$ in one pass, rather than advancing it one state per episode.
On the instances whose optimum is $1$ this makes \QUASARBR{} numerically exact within a single episode regardless of goal depth (Section~\ref{sec:experiments}), which we state and prove as Theorem~\ref{thm:onepass} in Sup.\ mat.~\ref{app:br}.

\section{Proof of Convergence}
\label{sec:convergence}
We establish that the iterates of \QUASARN{} (Algorithm~\ref{alg:qlearning}) converge to $\Qstar$ almost surely (Theorem~\ref{thm:main}), building the argument in stages.

A learner that knew $\P$ could iterate the exact right-hand side of~\eqref{def:qboe}. We first cast this as fixed-point iteration of an operator $\H$ whose unique fixed point is $\Qstar$, and show $\H$ is a \emph{contraction} toward $\Qstar$ in a weighted norm (Property~\ref{prop:contraction}), so the exact iteration converges geometrically. Our learner does not know $\P$: it replaces the exact right-hand side by a single sampled successor (Line~12), making each update one exact step of $\H$ plus a mean-zero perturbation. A classical stochastic-approximation result (Lemma~\ref{lem:tsitsiklis}) shows the perturbation washes out, provided every pair is visited infinitely often (Property~\ref{prop:exploration}) and the step sizes obey the Robbins--Monro rate (Assumption~\ref{asm:lr}). This decomposition of a deterministic contraction perturbed by mean-zero noise is the classical route for discounted Q-learning, but it is closed to reachability. A known factor $\gamma < 1$ makes the discounted operator contract in the supremum norm, whereas reachability is the undiscounted case $\gamma = 1$ where it does not contract at all, and the stochastic shortest path analysis does not transfer either, as its contraction presumes a \emph{proper} policy reaching the target with probability one that reachability does not grant. Our contribution is to recover the contraction from within the MDP, with the expected hitting time supplying the role a discount factor plays classically. Full proofs are in the sup.\ mat.
\subsection{The Bellman Optimality Operator}
One exact step replaces each entry of $Q$ by its one-step lookahead under the true dynamics, averaging successor values as~\eqref{def:qboe} prescribes; the sampled successor on Line~12 is a one-draw estimate of this lookahead. Formally, it applies the \emph{Bellman optimality operator} $\H : \mathbb{R}^{\Snt \times \A} \to \mathbb{R}^{\Snt \times \A}$, defined for $(\SApair{}) \in \Snt \times \A$ by
\begin{equation}\label{def:bellman_operator}
\begin{split}
    (\H Q)(s, a) &\coloneqq \sum_{s' \in \ST} \P(s' \mid \SApair{}) \\
    &\quad + \sum_{s' \in \Snt} \P(s' \mid \SApair{}) \max_{a' \in \A(s')} Q(s', \, a').
\end{split}
\end{equation}
Comparing with~\eqref{def:qboe}, $\Qstar$ is exactly the fixed point of $\H$. Convergence thus amounts to driving $Q_t$ to this fixed point.
\subsection{Contraction under a Hitting-Time Weight}
For fixed-point iteration to converge, $\H$ must contract toward $\Qstar$. It does not contract in the supremum norm, but a suitably \emph{weighted} norm restores contraction, provided the agent cannot postpone termination indefinitely.
\Theoremspace
\begin{definition}[Weighted maximum norm]
    For a strictly positive weight $w : \S \times \A \to (0, \infty)$, the \emph{weighted maximum norm} of $Q \in \mathbb{R}^{\S \times \A}$ is $\wnorm{Q} \coloneqq \maxSAnt |Q(s,a)| / w(s,a)$.
\end{definition}
The weight is the worst-case expected time to termination: for $s \in \Snt$, $T_s \coloneqq \sup_{\pi} \E^\pi[T \mid s_0 = s]$, where $T = \inf\{t \ge 0 : s_t \in \Sterm\}$ is the hitting time of $\Sterm$. A state from which termination is expected only after many steps is pinned down more slowly, so weighting each state's error by $T_s$ places all states on a common timescale. The no-EC assumption is what keeps this weight finite.
\Theoremspace
\begin{property}[Almost-sure termination]\label{prop:termination}
    Under Assumption~\ref{asm:finite_hitting}, for every policy $\pi$ and initial state $s$ the hitting time satisfies $T < \infty$ almost surely. Moreover, $1 \le T_s \le |\Snt|/p < \infty$ for some constant $p > 0$.
\end{property}
With no end component among the non-terminal states, the agent cannot cycle within $\Snt$ forever, so from any non-terminal state a terminal state is reached within $|\Snt|$ steps with probability at least some $p > 0$. A geometric-tail bound then yields both claims (the sup.\ mat.).
Finiteness of $T_s$ is, in turn, exactly what makes $\H$ a contraction.
\Theoremspace
\begin{property}[Contraction]\label{prop:contraction}
    Under Assumption~\ref{asm:finite_hitting}, with weight $w(\SApair{}) \coloneqq T_s$, for all $Q \in [0,1]^{\Snt \times \A}$,
    \[
        \wnorm{\H Q - Q^*} \;\le\; \beta\, \wnorm{Q - Q^*},
    \]
    where $\beta = 1 - 1/\max_{s \in \Snt} T_s \in [0,1)$.
\end{property}
The crux is the one-step inequality $\sum_{s' \in \Snt} \P(s' \mid \SApair{})\, T_{s'} \le T_s - 1$ (leaving $s$ consumes one step of the budget), which bounds the per-pair contraction factor by $(T_s - 1)/T_s$. Maximizing over states gives $\beta = 1 - 1/\max_s T_s < 1$ (the sup.\ mat.). The contraction thus rests entirely on $\max_{s \in \Snt} T_s < \infty$, which is the sole place the no-EC assumption (Assumption~\ref{asm:finite_hitting}) enters the analysis. An end component among the non-terminal states would let the agent avoid termination forever, driving $T_s$ (and with it $\beta$) to $1$ and destroying the contraction. In this sense the hitting time is the reachability analogue of a discount factor, and $1/\max_s T_s$ the effective discount gap, intrinsic to the MDP rather than a constant the learner chooses.
\subsection{Q-Learning as Stochastic Approximation}
The learner cannot apply $\H$ exactly, only through a single sampled successor per step. Writing the update as one exact step of $\H$ plus a mean-zero noise term casts it as a \emph{stochastic-approximation} recursion. At the visited pair $(\SApair{})$,
\begin{equation}\label{eq:sa-update}
    Q_{t+1} = (1 - \alpha_t)\,Q_t + \alpha_t\, \H Q_t + \alpha_t\, \err_t,
\end{equation}
where $\err_t \coloneqq \max_{a' \in \A(s')} Q_t(s', a') - (\H Q_t)(\SApair{})$. Because $s'$ is drawn from the true kernel, $\E[\err_t \mid \SApair{t}] = 0$, so the noise is unbiased. Since every $Q_t \in [0,1]$ (each update is a convex combination of values in $[0,1]$), it is bounded, $|\err_t| \le 1$. Intuitively, if a pair is revisited without end while the steps shrink at the right rate, these mean-zero errors cancel in aggregate and only their zero mean survives. Two conditions make this precise. The first is a consequence of $\epsilon$-greedy exploration, the second a standard step-size assumption.
\Theoremspace
\begin{property}[Sufficient exploration]\label{prop:exploration}
    Under Assumption~\ref{asm:finite_hitting}, under the $\epsilon$-greedy policy with $\epsilon > 0$, every pair $(\SApair{}) \in \SAnt$ is visited infinitely often almost surely~\cite{Aminetal2021}.
\end{property}
\Theoremspace
\begin{assumption}[Robbins--Monro step sizes]\label{asm:lr}
    For every $(\SApair{}) \in \SAnt$, almost surely $\sum_{n} \alpha_n(\SApair{}) = \infty$ and $\sum_{n} \alpha_n(\SApair{})^2 < \infty$.
\end{assumption}
Two indices are in play and it is worth separating them. The \emph{schedule} $\alpha_n(\SApair{})$ is indexed by the visit count $n$ and is what Line~12 applies as $\alpha_{N(\SApair{})}(\SApair{})$. The \emph{applied} step size at global time $t$, written $\alpha_t$ in~\eqref{eq:sa-update} and in Lemma~\ref{lem:tsitsiklis}, equals $\alpha_{N_t(\SApair{})}(\SApair{})$ at the pair visited at time $t$ and $0$ at every other pair. Since the vanishing terms contribute nothing, the two sums agree, $\sum_t \alpha_t(\SApair{}) = \sum_n \alpha_n(\SApair{})$, and likewise for the squares. This is why Assumption~\ref{asm:lr} can be stated on the schedule alone while Lemma~\ref{lem:tsitsiklis} consumes it in global time.

Together these place~\eqref{eq:sa-update} within the scope of a stochastic-approximation lemma, adapted from~\cite{tsitsiklis1994asynchronous}. It is worth noting that the two conditions interlock. The step-size sums of Assumption~\ref{asm:lr} advance only when the pair is visited, so they are realizable precisely because Property~\ref{prop:exploration} visits every pair infinitely often. Property~\ref{prop:exploration} in turn requires Assumption~\ref{asm:finite_hitting}: episodes restart at $\snot$ only upon reaching $\Sterm$, and an end component among the non-terminal states would trap the $\epsilon$-greedy walk forever, so that pairs outside it are visited finitely often. It also uses the convention of Section~\ref{sec:prelims} that every state is reachable from $\snot$; at an unreachable pair the sum $\sum_n \alpha_n(\SApair{})$ would be empty and Assumption~\ref{asm:lr} unsatisfiable.
\Theoremspace
\begin{lemma}[{\cite{tsitsiklis1994asynchronous}}]\label{lem:tsitsiklis}
    Let $\{\mathcal{F}_t\}$ be a filtration, let $\alpha_t, \err_t, W_t$ be scalar random variables with $\alpha_t$ and $W_t$ being $\mathcal{F}_t$-measurable and $\err_t$ being $\mathcal{F}_{t+1}$-measurable, and let $B \ge 0$ be a constant such that, with probability $1$, $\E[\err_t \mid \mathcal{F}_t] = 0$, $\E[\err_t^2 \mid \mathcal{F}_t] \le B$, $\alpha_t \in [0,1]$, $\sum_t \alpha_t = \infty$, and $\sum_t \alpha_t^2 < \infty$. If $W_{t+1} = (1 - \alpha_t) W_t + \alpha_t \err_t$, then $W_t \to 0$ almost surely.
\end{lemma}
\subsection{Convergence}
We now combine the two ingredients, the contraction of Property~\ref{prop:contraction} and the noise of Lemma~\ref{lem:tsitsiklis}, into the main result.
\Theoremspace
\begin{theorem}\label{thm:main}
    Under Assumptions~\ref{asm:finite_hitting} and~\ref{asm:lr}, the Q-learning iterates satisfy $Q_t(\SApair{}) \to Q^*(\SApair{})$ almost surely for every $(\SApair{}) \in \SAnt$.
\end{theorem}
\noindent\emph{Proof sketch.}
We track the error $\Delta_t = Q_t - \Qstar$ through ever-smaller ceilings. Subtracting $\Qstar$ from~\eqref{eq:sa-update} gives $\Delta_{t+1} = (1-\alpha_t)\Delta_t + \alpha_t(\H Q_t - \Qstar) + \alpha_t \err_t$, whose driving terms are the two ingredients above: the contraction term, with weighted norm at most $\beta\,\wnorm{\Delta_t}$ (Property~\ref{prop:contraction}), and the noise term, averaging to $0$ (Lemma~\ref{lem:tsitsiklis}). Fix $\delta$ with $\beta + 2\delta < 1$ and set $X_k = (\beta + 2\delta)^k \to 0$. If $\wnorm{\Delta_t} \le X_k$ from some time on, the contraction pulls the error toward $\beta X_k$ while the accumulated noise (an auxiliary sequence to which Lemma~\ref{lem:tsitsiklis} applies) eventually stays below $\delta X_k$, pushing $\wnorm{\Delta_t}$ below $X_{k+1}$. Since $X_k \to 0$, $\wnorm{\Delta_t} \to 0$. The full induction is in the sup.\ mat. \hfill$\square$

\subsection{From Value Convergence to the Asymptotic Guarantee}
Theorem~\ref{thm:main} is a statement about the \emph{estimates}, whereas Definition~\ref{def:asymptotic} is a statement about the \emph{induced policies}. The two are not interchangeable, and the bridge between them is where Assumption~\ref{asm:finite_hitting} is used a second time. We make the induced policy explicit: let
    $\pi_t(s) \;\in\; \argmax_{a \in \A(s)} Q_t(s,a), \qquad s \in \Snt$,
be any greedy policy with respect to the current estimate, ties broken arbitrarily.
\Theoremspace
\begin{corollary}\label{cor:asymptotic}
Under Assumptions~\ref{asm:finite_hitting} and~\ref{asm:lr}, Algorithm~\ref{alg:qlearning} converges asymptotically in the sense of Definition~\ref{def:asymptotic}: $V^{\pi_t}(\snot) \to \Vstar(\snot)$ almost surely. In fact the convergence is eventually exact, i.e., almost surely $V^{\pi_t}(\snot) = \Vstar(\snot)$ for all large $t$.
\end{corollary}
The argument has two steps---greediness eventually selects only $\Qstar$-optimal actions, and every $\Qstar$-greedy policy is optimal (again via Assumption~\ref{asm:finite_hitting})---and is given in full in the sup.\ mat.

We stress that Corollary~\ref{cor:asymptotic} concerns the \emph{greedy} policies $\{\pi_t\}$, not the $\epsilon$-greedy behavior policy that Algorithm~\ref{alg:qlearning} follows while learning. The distinction is material: with the fixed $\epsilon > 0$ of Line~1, the behavior policy plays a uniformly random action with probability $\epsilon$ forever, so its value stays bounded away from $\Vstar(\snot)$ and Definition~\ref{def:asymptotic} would fail for it. It is the estimates, and the policies they induce, that converge.

%% file: Sections/empirical.tex
\section{Experiments}
\label{sec:experiments}

We ask whether the model-free route yields a practical algorithm and how it compares against the model-based state of the art. Our evaluation answers the first affirmatively and, on the second, finds a substantial margin in our favor.

\subsection{Implementation Details}
\label{sec:exp-impl}

We evaluate two versions of \QUASAR{}, \QUASARN{} and \QUASARBR{}. 
We implement both versions in Python with NumPy. Each maintains a single dense $Q$-table and so occupies $O(|\S|\,|\A|)$ memory. We use the per-pair step size $\alpha_t = 1/N(\SApair{})$, i.e., the inverse visit count, and the uniform exploration policy $\epsilon = 1$, i.e., we sample actions uniformly from $\A(s)$. An episode ends on reaching $\Sterm$, or is truncated after $H = 10\,|\S \setminus \ST|$ steps.

\subsection{Experimental Setup}
\label{sec:exp-setup}

We evaluate on all nine models from the standardized Quantitative Verification Benchmark Set~\cite{qvbs2019} (Table~\ref{tab:main}). Three of the nine (\textsc{csma}, \textsc{pacman}, \textsc{zeroconf}) contain non-terminal MECs and hence fall outside Assumption~\ref{asm:finite_hitting}. For these, marked $\dagger$ throughout, we first apply the standard MEC-quotient procedure~\cite{deAlfaro1997} to obtain an MEC-free MDP.

\begin{table*}[!t]
\centering
\caption{Benchmark MDPs and head-to-head results. $d$ = BFS goal depth of $\snot$. $\dagger$ = trained on the MEC-quotient, sizes reported post-collapse. Final error is measured after 36\,h, and ``samples to $<10^{-2}$ error'' is the number of samples needed to sustain that error, with --- marking a benchmark on which it was never sustained. All entries are medians over 5 seeds.}
\label{tab:main}
\small
\setlength{\tabcolsep}{4pt}
\begin{tabular}{@{}lrrrcrrrrrr@{}}
\toprule
& \multicolumn{4}{c}{Benchmark} & \multicolumn{2}{c}{\QUASARBR{}} & \multicolumn{2}{c}{\QUASARN{}} & \multicolumn{2}{c}{\SPAC{}} \\
\cmidrule(lr){2-5}\cmidrule(lr){6-7}\cmidrule(lr){8-9}\cmidrule(lr){10-11}
MDP & $|S|$ & $|S{\times}A|$ & $d$ & $V^*(\snot)$
  & \shortstack[r]{Final\\error}
  & \shortstack[r]{Samples to\\$<10^{-2}$ error}
  & \shortstack[r]{Final\\error}
  & \shortstack[r]{Samples to\\$<10^{-2}$ error}
  & \shortstack[r]{Final\\error}
  & \shortstack[r]{Samples to\\$<10^{-2}$ error} \\
\midrule
ij.3               &      7 &     12 &  2 & $1$               & $9{\times}10^{-13}$  & 5        & $1.3{\times}10^{-4}$ & 147 & $6.9{\times}10^{-3}$ & $2.8{\times}10^{7}$ \\
philosophers.3     &    956 & 3\,342 &  4 & $1$               & $0$                  & 15       & $1.5{\times}10^{-5}$ & $3.2{\times}10^{6}$ & $2.0{\times}10^{-2}$ & ---                 \\
rabin.3            & 27\,766 & 45\,636 &  4 & $1$              & $0$                 & 7        & $1.7{\times}10^{-5}$ & $3.2{\times}10^{6}$ & $7.2{\times}10^{-3}$ & $9.3{\times}10^{8}$ \\
ij.10              & 1\,023 & 5\,120 &  9 & $1$               & $2{\times}10^{-11}$  & 68       & $5.5{\times}10^{-1}$ & --- & $1.1{\times}10^{-1}$ & ---                 \\
consensus.2        &    272 &    400 & 12 & $1$               & $6{\times}10^{-12}$  & 100      & $7.9{\times}10^{-1}$ & --- & $3.9{\times}10^{-2}$ & ---                 \\
firewire           &    623 &    714 & 76 & $1$               & $0$                  & 215      & $1.0$                & --- & $1.2{\times}10^{-1}$ & ---                 \\
pacman$^\dagger$   &    498 &    552 &  7 & $0.551$           & $1.0{\times}10^{-5}$ & $10^{4}$ & $1.6{\times}10^{-4}$ & $1.5{\times}10^{7}$ & $6.2{\times}10^{-2}$ & ---                 \\
zeroconf$^\dagger$ &    670 &    825 & 17 & $2{\cdot}10^{-5}$ & $2.0{\times}10^{-5}$ & 1        & $2.0{\times}10^{-5}$ & 1   & $6.5{\times}10^{-2}$ & ---                 \\
csma$^\dagger$     & 1\,038 & 1\,054 & 47 & $0.5$             & $5.0{\times}10^{-4}$ & $10^{7}$ & $5.0{\times}10^{-1}$ & --- & $4.1{\times}10^{-2}$ & ---                 \\
\bottomrule
\end{tabular}
\end{table*}

\begin{figure*}[!t]
\centering
\includegraphics[width=0.42\textwidth]{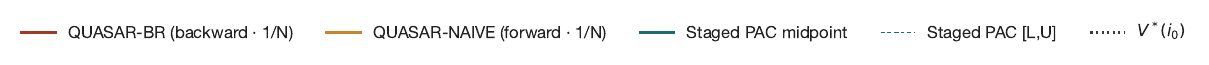}\\[0pt]
\begin{minipage}{0.265\textwidth}\centering
  \includegraphics[width=\linewidth]{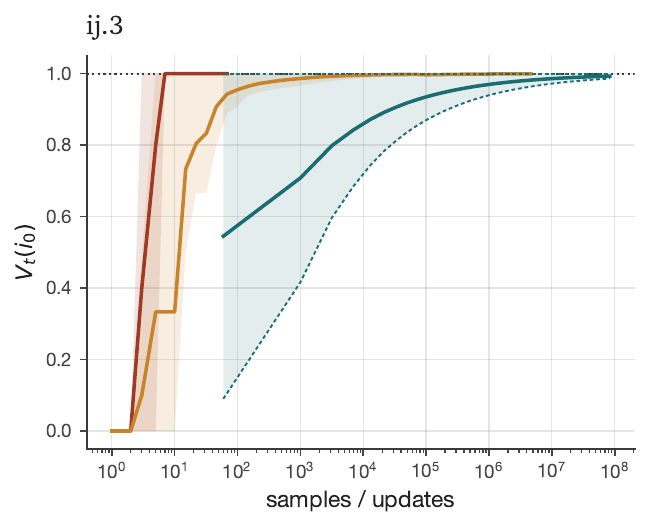}
\end{minipage}\hfill
\begin{minipage}{0.265\textwidth}\centering
  \includegraphics[width=\linewidth]{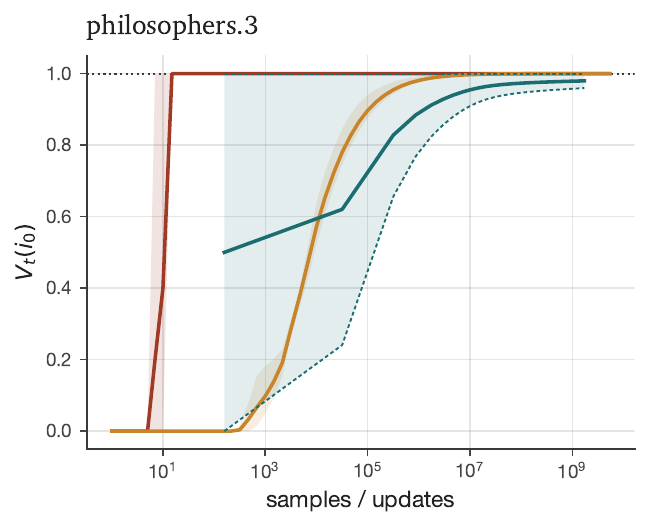}
\end{minipage}\hfill
\begin{minipage}{0.265\textwidth}\centering
  \includegraphics[width=\linewidth]{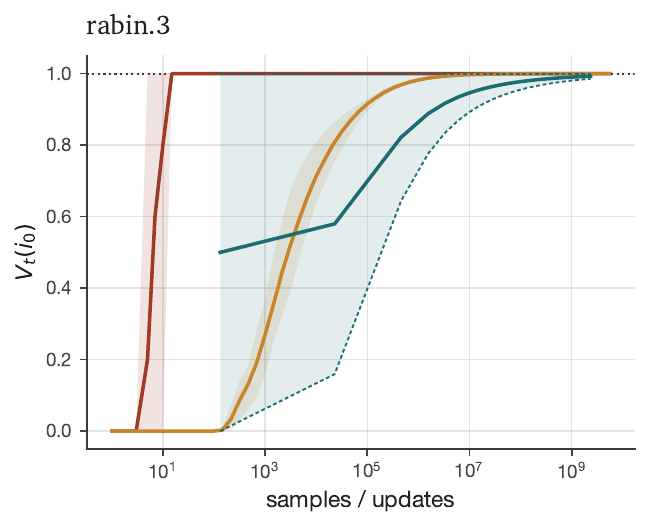}
\end{minipage}\\[0pt]
\begin{minipage}{0.265\textwidth}\centering
  \includegraphics[width=\linewidth]{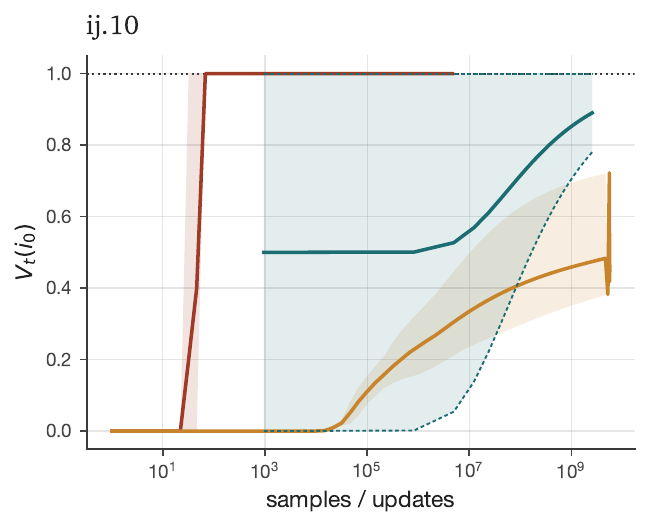}
\end{minipage}\hfill
\begin{minipage}{0.265\textwidth}\centering
  \includegraphics[width=\linewidth]{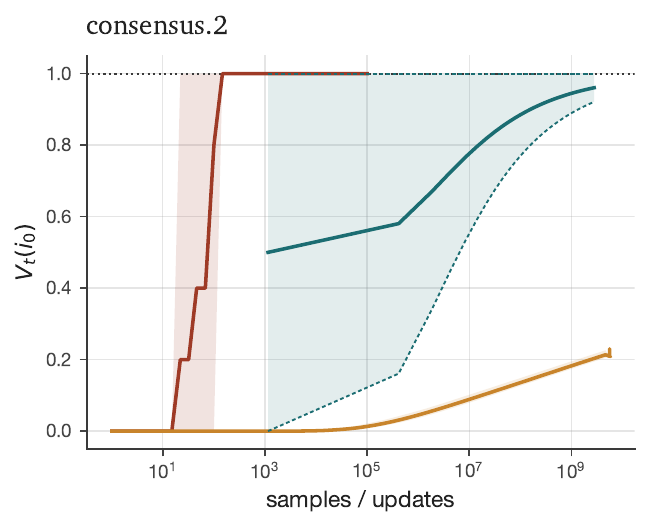}
\end{minipage}\hfill
\begin{minipage}{0.265\textwidth}\centering
  \includegraphics[width=\linewidth]{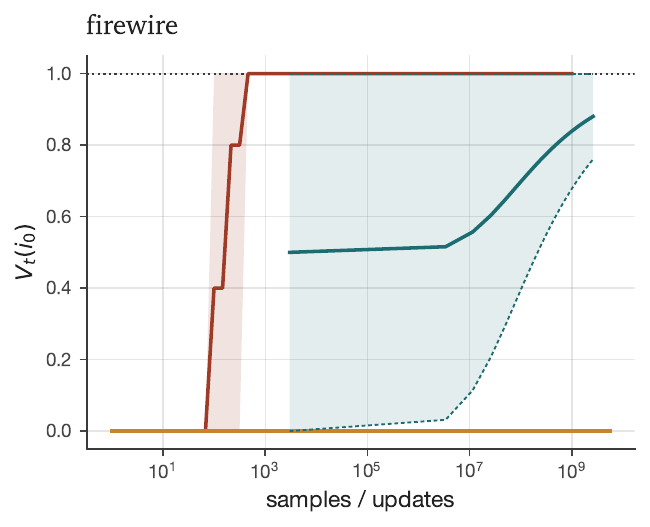}
\end{minipage}\\[0pt]
\begin{minipage}{0.265\textwidth}\centering
  \includegraphics[width=\linewidth]{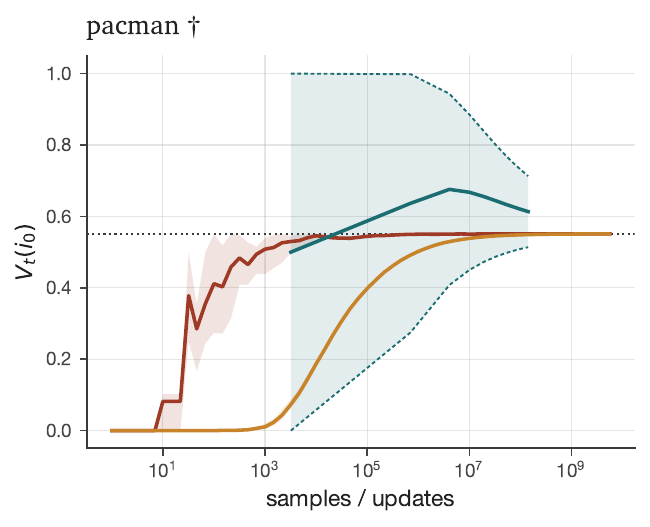}
\end{minipage}\hfill
\begin{minipage}{0.265\textwidth}\centering
  \includegraphics[width=\linewidth]{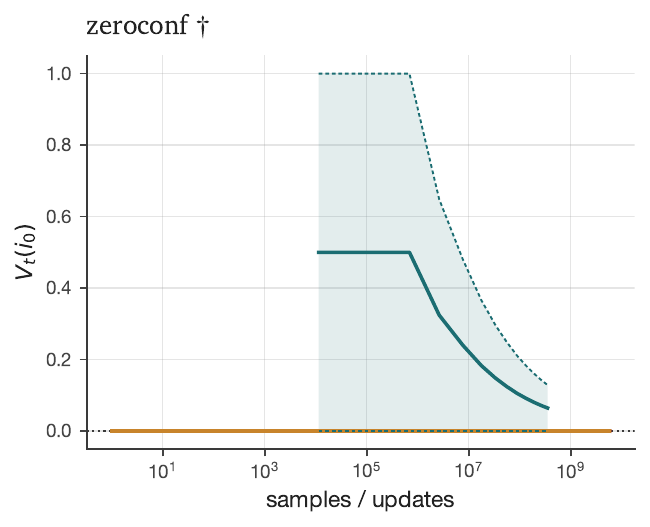}
\end{minipage}\hfill
\begin{minipage}{0.265\textwidth}\centering
  \includegraphics[width=\linewidth]{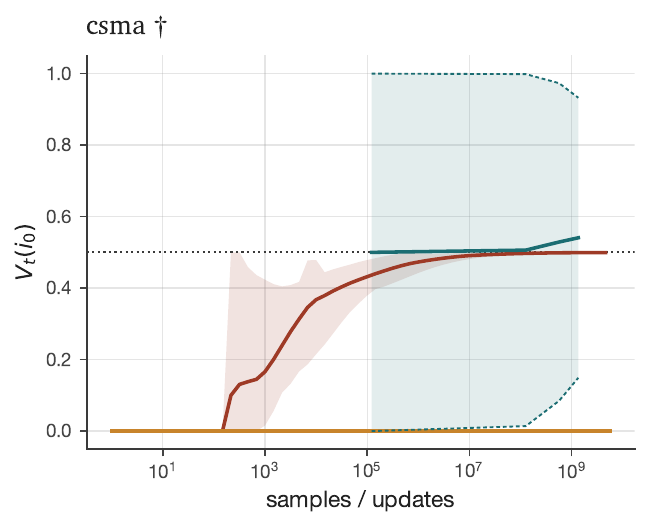}
\end{minipage}
\caption{Value estimate at the initial state vs.\ samples on all nine
benchmarks (seed mean; shaded band = min--max range over seeds; 5
seeds per series). \QUASARBR{} and \QUASARN{} 
against the \SPAC{} baseline: the
shaded band with dashed edges is the baseline's own output, its
$[L,U]$ bounds; the solid curve is the midpoint
$\tfrac{1}{2}(L{+}U)$, our minimax point readout of the interval; horizontal line = $V^*(\snot)$.
All three series run a 36\,h wall.
$\dagger$ = MEC-quotient input.}
\label{fig:headline}
\end{figure*}

To measure progress, we precompute $V^*(\snot)$ on every benchmark and record the \emph{initial-state value error}, which \QUASAR{} reports as $\verr = |V_t(\snot) - V^*(\snot)|$ and \SPAC{} as the error of its bound midpoint, $|\tfrac{1}{2}(L{+}U) - V^*(\snot)|$. We report this error against the number of samples consumed.
We compare against the prior algorithm with an asymptotic guarantee for reachability \SPAC{}~\cite{palasamudram2026asymptotic}. Comparison to all other methods without such a guarantee offers no meaningful yardstick for a convergence result. \SPAC{} grows a partial model by simulation and runs bounded value iteration with bounds $L(\snot) \le V^*(\snot) \le U(\snot)$, and we run their open-source implementation under their per-benchmark hyperparameter configurations~\cite{palasamudram2026code}.

We run each algorithm on each benchmark with $5$ random seeds under a $36$\,h wall-clock timeout, one task per run on a cluster with $2$ CPU cores and $16$\,GB RAM each. \SPAC{} uses its authors' staged stopping rule until the timeout. We stop a \QUASAR{} run once $\qerr < 10^{-3}$ has held for three consecutive checkpoints on a log-spaced grid.

\subsection{Results}
\label{sec:exp-main}

Table~\ref{tab:main} and Figure~\ref{fig:headline} report the final error of each learner and the number of samples it needs to sustain the initial-state error below $10^{-2}$.

\QUASARBR{} converges on all nine benchmarks within the $36$\,h budget, and its margin over the model-based state of the art is decisive. On the six instances with $V^*(\snot) = 1$ it reaches numerically exact values, error at most $2{\times}10^{-11}$, within $5$ to $215$ samples, fewer updates than a single return path to the goal, and its final error is below \SPAC{}'s on every benchmark under identical budgets. In particular, \SPAC{} sustains error below $10^{-2}$ on only two of the nine, and even there it spends $2.8{\times}10^{7}$ samples on \textsc{ij.3} and $9.3{\times}10^{8}$ on \textsc{rabin.3} where \QUASARBR{} needs $5$ and $7$. On the remaining seven it halts at the timeout with error between $2{\times}10^{-2}$ and $1.2{\times}10^{-1}$ and its bounds still open (Figure~\ref{fig:headline}). This understates the gap on \textsc{csma}$^\dagger$, whose initial bounds $[0,1]$ have midpoint $0.5 = V^*(\snot)$, so the baseline's reported error starts near zero at full interval width and our reported speedup is, if anything, conservative.

The two versions of \QUASAR{} part ways as the goal depth $d$ grows. \QUASARN{} keeps pace on the shallow benchmarks but fails outright on the deepest, leaving the error at $1.0$ on \textsc{firewire} ($d = 76$) after $2.8{\times}10^{8}$ samples even though every episode reaches the goal, where \QUASARBR{} reaches error $0$ within $215$. The update ordering is thus not a minor optimization but the difference between converging and not, and every \QUASAR{} entry in Table~\ref{tab:main} is accordingly a \QUASARBR{} entry.

A single mechanism accounts for both gaps. Under the online $1/N$ schedule each $Q(\SApair{})$ averages every target it has seen, so a corrected value advances one state per episode and its propagation time is exponential in $d$, whereas a reverse sweep takes $\alpha = 1$ on a first visit and runs exact value iteration along the sampled path in one pass. That same sweep costs essentially one goal-reaching episode, whereas \SPAC{} must first estimate the model to uniform accuracy at a cost scaling with the model size and with $1/p_{\min}$.

Finally, these findings are not an artifact of averaging over seeds. Sup.\ mat.~\ref{app:perseed} disaggregates them, and every finding holds on every seed.

%% file: Sections/conclusion.tex
\section{Conclusion}
\label{sec:conclusion}

We presented \QUASAR{}, the first model-free algorithm that learns reachability objectives with asymptotic guarantees on the no-EC fragment of MDPs, matching the guarantee of model-based methods with only $O(|\S|\,|\A|)$ memory while improving substantially on their empirical state of the art across a standardized benchmark suite. This is substantial progress towards making specification-guided RL practical. 
While the no-MEC assumption is a building block to learning on arbitrary MDPs, addressing the later is the most pressing next step.

%% file: Sections/appendix.tex
\section{Per-Seed Trajectories}
\label{app:perseed}

Figures~\ref{fig:perseed-vt} and~\ref{fig:perseed-qerr} disaggregate the headline results of $\S$\ref{sec:exp-main} to individual seeds. Each panel carries a second (top) axis converting samples to wall time via the panel's median sampling rate, which is near-constant within a run.

\begin{figure*}[t]
\centering
\includegraphics[width=.55\textwidth]{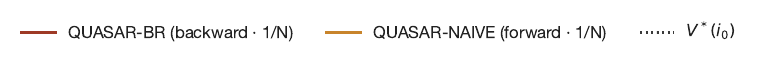}\\[1pt]
\begin{minipage}{0.32\textwidth}\centering
  \includegraphics[width=\linewidth]{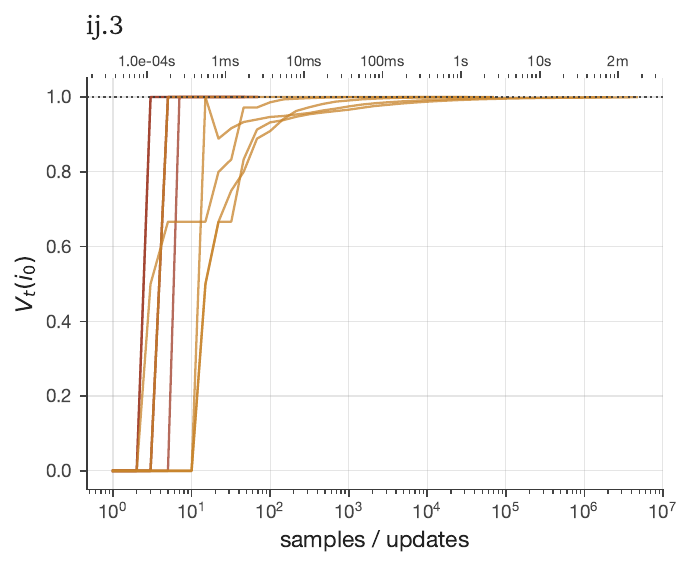}
\end{minipage}\hfill
\begin{minipage}{0.32\textwidth}\centering
  \includegraphics[width=\linewidth]{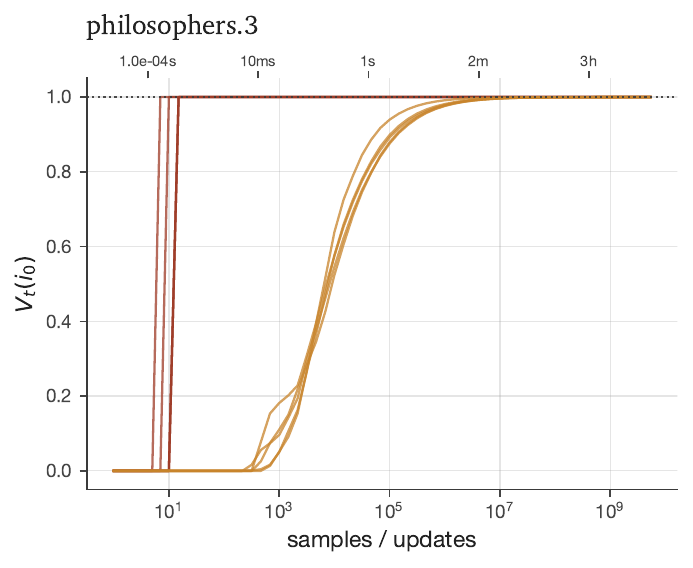}
\end{minipage}\hfill
\begin{minipage}{0.32\textwidth}\centering
  \includegraphics[width=\linewidth]{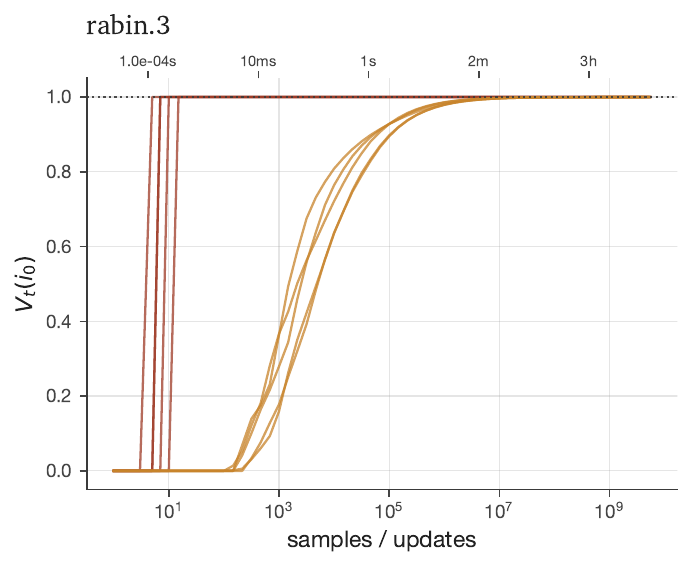}
\end{minipage}\\[2pt]
\begin{minipage}{0.32\textwidth}\centering
  \includegraphics[width=\linewidth]{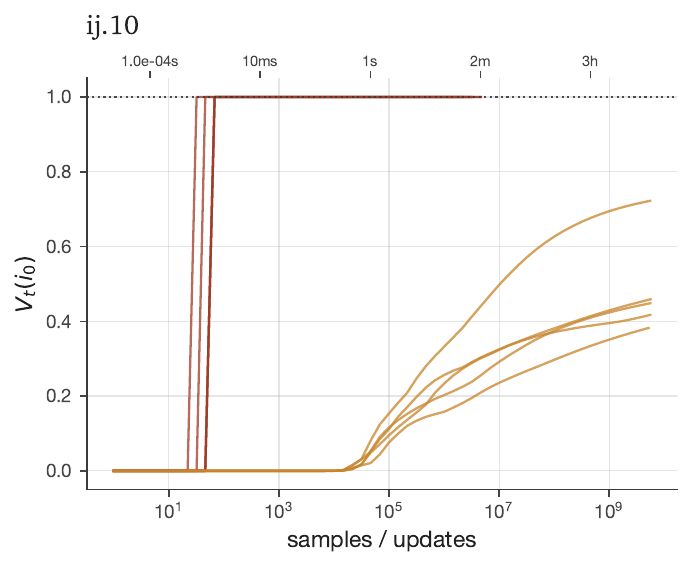}
\end{minipage}\hfill
\begin{minipage}{0.32\textwidth}\centering
  \includegraphics[width=\linewidth]{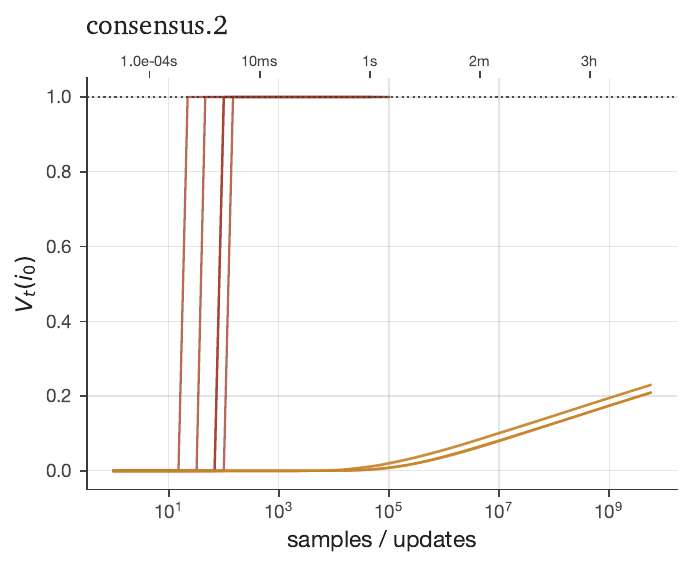}
\end{minipage}\hfill
\begin{minipage}{0.32\textwidth}\centering
  \includegraphics[width=\linewidth]{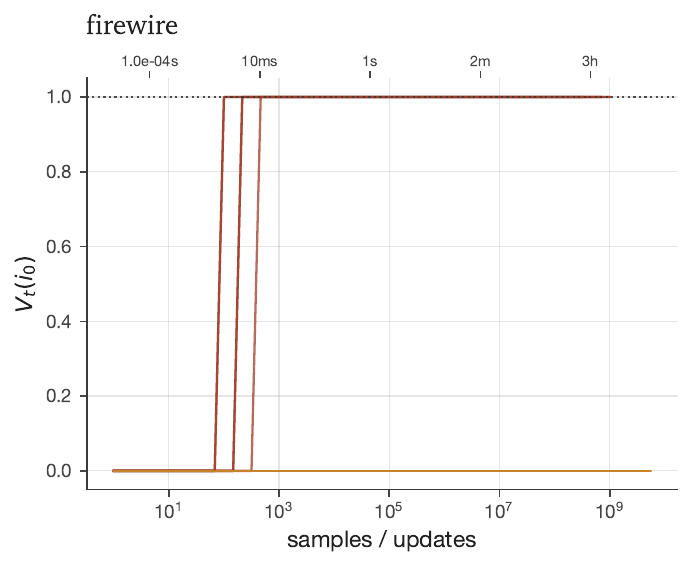}
\end{minipage}\\[2pt]
\begin{minipage}{0.32\textwidth}\centering
  \includegraphics[width=\linewidth]{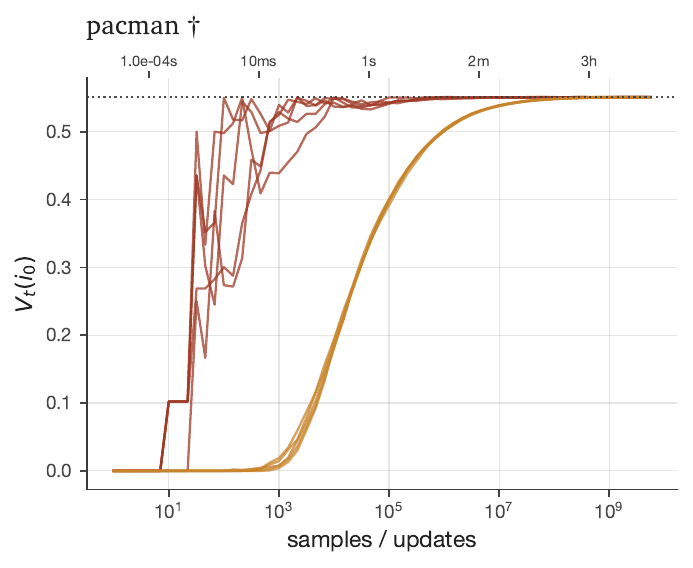}
\end{minipage}\hfill
\begin{minipage}{0.32\textwidth}\centering
  \includegraphics[width=\linewidth]{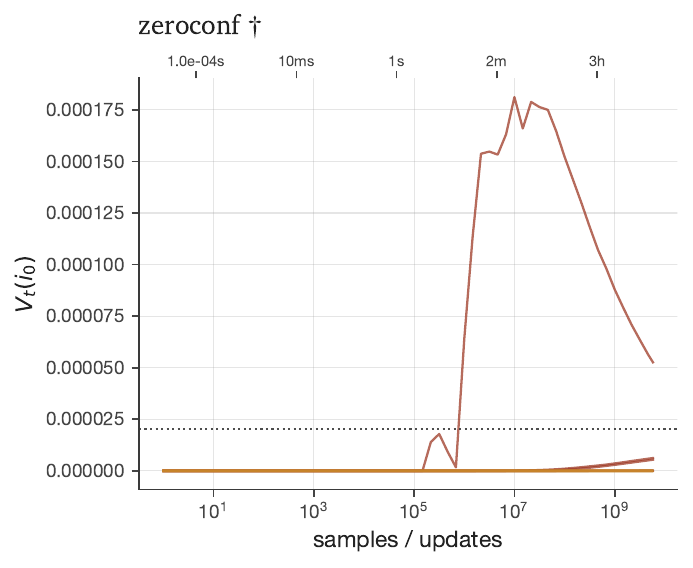}
\end{minipage}\hfill
\begin{minipage}{0.32\textwidth}\centering
  \includegraphics[width=\linewidth]{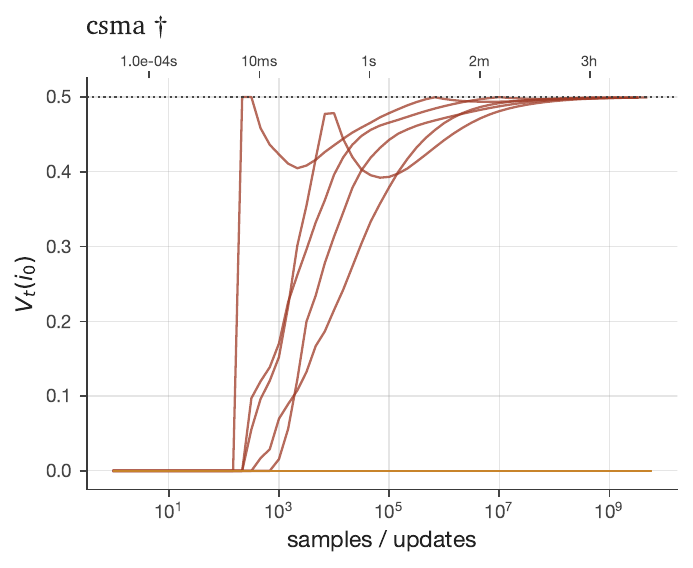}
\end{minipage}
\caption{Per-seed value estimate at the initial state vs.\ samples:
one line per run, 10 runs per panel --- 5 \QUASARBR{} seeds and
5 \QUASARN{} seeds. Bottom axis:
samples (log). Top axis: wall time (log), aligned via the \QUASARBR{}
runs' median sampling rate (samples/sec is near-constant within a run;
indicative only for \QUASARN{}, which paces differently).}
\label{fig:perseed-vt}
\end{figure*}

\begin{figure*}[t]
\centering
\includegraphics[width=.55\textwidth]{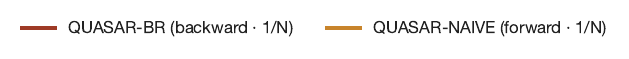}\\[1pt]
\begin{minipage}{0.32\textwidth}\centering
  \includegraphics[width=\linewidth]{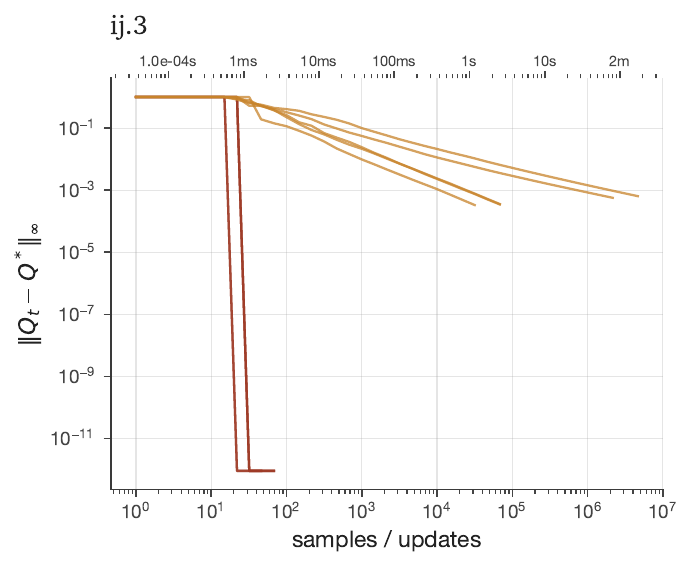}
\end{minipage}\hfill
\begin{minipage}{0.32\textwidth}\centering
  \includegraphics[width=\linewidth]{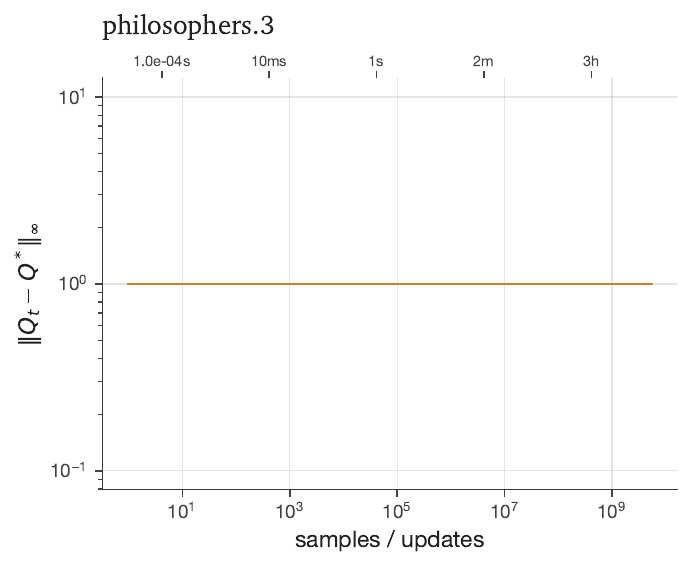}
\end{minipage}\hfill
\begin{minipage}{0.32\textwidth}\centering
  \includegraphics[width=\linewidth]{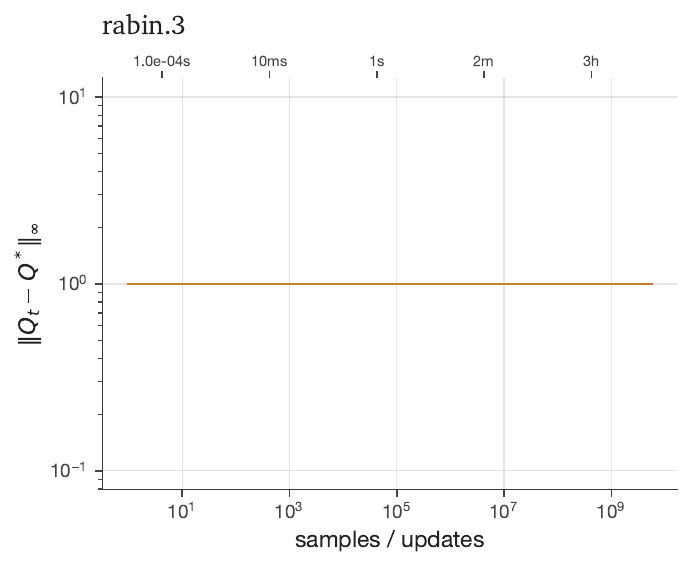}
\end{minipage}\\[2pt]
\begin{minipage}{0.32\textwidth}\centering
  \includegraphics[width=\linewidth]{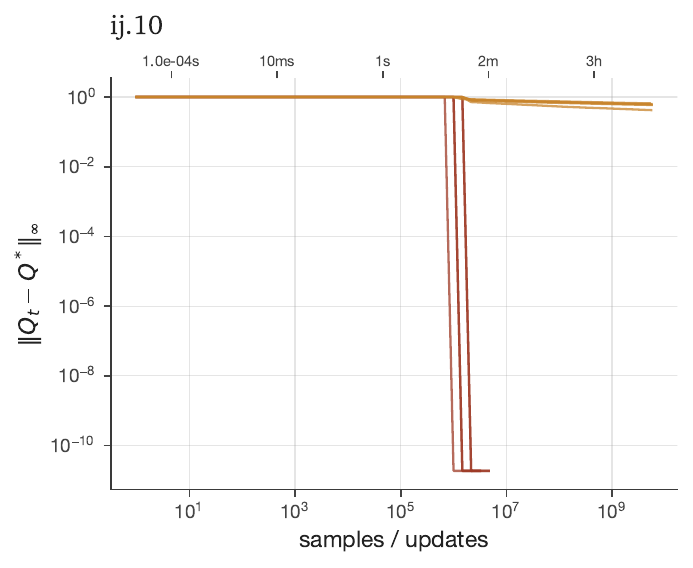}
\end{minipage}\hfill
\begin{minipage}{0.32\textwidth}\centering
  \includegraphics[width=\linewidth]{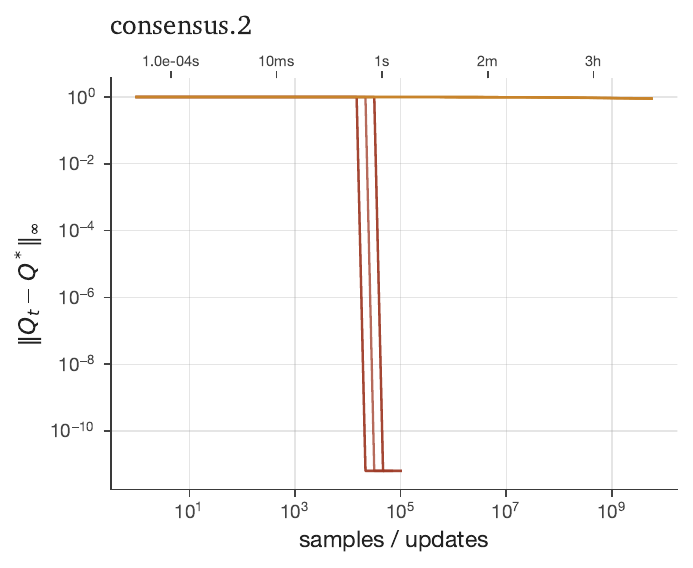}
\end{minipage}\hfill
\begin{minipage}{0.32\textwidth}\centering
  \includegraphics[width=\linewidth]{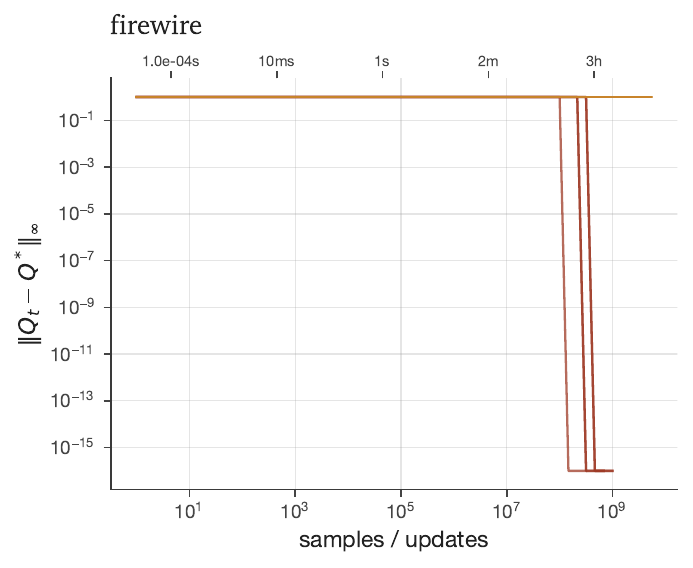}
\end{minipage}\\[2pt]
\begin{minipage}{0.32\textwidth}\centering
  \includegraphics[width=\linewidth]{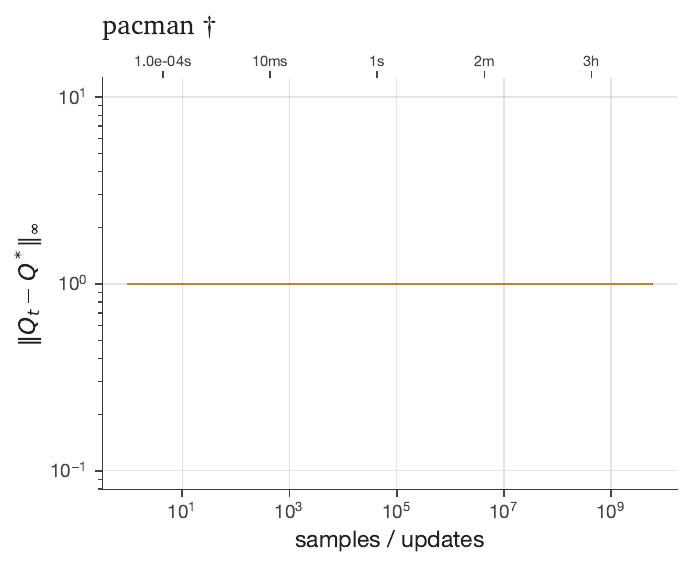}
\end{minipage}\hfill
\begin{minipage}{0.32\textwidth}\centering
  \includegraphics[width=\linewidth]{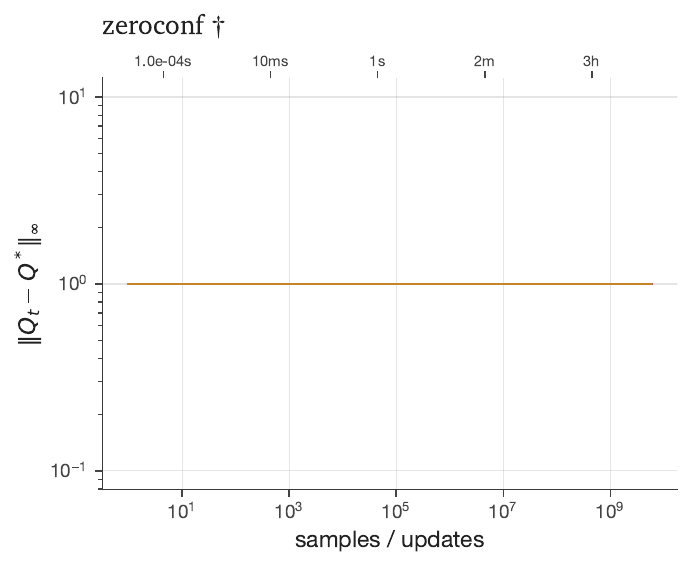}
\end{minipage}\hfill
\begin{minipage}{0.32\textwidth}\centering
  \includegraphics[width=\linewidth]{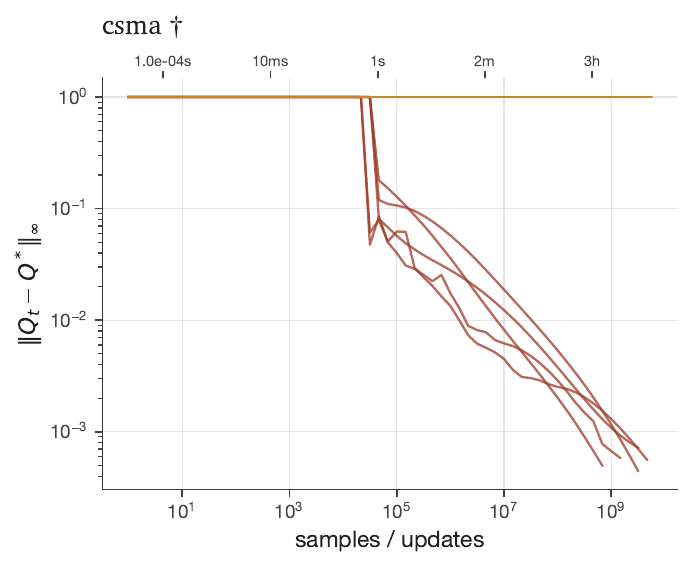}
\end{minipage}
\caption{Per-seed full-table sup-norm $\|Q_t - Q^*\|_\infty$, same runs and
colors as Figure~\ref{fig:perseed-vt}.}
\label{fig:perseed-qerr}
\end{figure*}

\section{Omitted Proofs}
\label{app:proofs}

This sup.\ mat.\ contains the proofs deferred from the main text.

\begin{proof}[Proof of Property~\ref{prop:termination} (almost-sure termination)]
We focus on an initial state $s \in \Snt$, since the hitting time is trivially $0$ for $s \in \Sterm$. Let $n = |\Snt|$. We first claim that for any $s \in \Snt$ and any deterministic policy $\pi$, there is a finite path that reaches a terminal state with positive probability. Suppose not. Then there exist $s \in \Snt$ and $\pi$ such that the probability of reaching $\Sterm$ from $s$ is $0$. The induced Markov chain is then confined to $\Snt$, and, since $\Snt$ is finite, it contains a closed recurrent class $D \subseteq \Snt$, which is strongly connected. The pairs $\{(x, \pi(x)) : x \in D\}$ then form an end component within $\Snt$, contradicting Assumption~\ref{asm:finite_hitting}.

By the pigeonhole principle, the shortest such path from any $s \in \Snt$ has length at most $n$, so after $n$ steps some probability mass reaches $\Sterm$. Let $p > 0$ be the minimum probability of reaching a terminal state from any non-terminal state within $n$ steps. Then
\[
    P^\pi \{ T \geq k \cdot n \} \leq (1-p)^k,
\]
and letting $k \to \infty$ gives $P^\pi \{ T < \infty\} \geq 1 - \lim_{k \to \infty} (1-p)^k = 1$. As this holds for every deterministic $\pi$, and the best policy for avoiding termination is deterministic (finitely many states), $T < \infty$ almost surely under every policy.

Finally, the same geometric tail bounds the expected hitting time uniformly in $\pi$:
\[
    T^\pi_s = \sum_{t=0}^\infty \P^\pi(T > t \mid s_0 = s)
    \le n \sum_{k=0}^\infty (1-p)^k = \tfrac{n}{p},
\]
so $T_s = \sup_\pi T^\pi_s \le n/p < \infty$, while $T_s \ge 1$ since reaching $\Sterm$ from a non-terminal state takes at least one transition.
\end{proof}

\begin{proof}[Proof of Property~\ref{prop:contraction} (contraction)]
We first derive a generic contraction factor for an arbitrary weight $w$, then specialize to $w = T_s$. Since $\H Q^* = Q^*$ and the target terms of~\eqref{def:bellman_operator} cancel,
\[
\begin{aligned}
    &|\H Q(\SApair{}) - Q^*(\SApair{})| \\
    &= \Big|\sum_{s' \in \Snt} \P(s' \mid s, a) \Big(\max_{a' \in \A(s')} Q(s', a') \\
    &\qquad\quad - \max_{a'' \in \A(s')} Q^*(s', a'')\Big)\Big| \\
    &\leq \sum_{s' \in \Snt}\P(s' \mid s, a) \max_{a' \in \A(s')} |Q(s', a') - Q^*(s', a')| \\
    &\leq \sum_{s' \in \Snt}\P(s' \mid s, a) \max_{a' \in \A(s')} w(s', a') \cdot \wnorm{Q - Q^*},
\end{aligned}
\]
using $|\max_{a'} f - \max_{a'} g| \le \max_{a'}|f - g|$ and $|Q - Q^*| \le w\,\wnorm{Q-Q^*}$. Dividing by $w(s,a)$ and maximizing over $\SAnt$ gives $\wnorm{\H Q - Q^*} \le \beta \,\wnorm{Q - Q^*}$ with
\[
    \beta = \maxSAnt \frac{\sum_{s' \in \Snt}\P(s' \mid \SApair{}) \max_{a' \in \A} w(s', a')}{w(s,a)}.
\]

We now take $w(\SApair{}) \coloneqq T_s$, finite by Property~\ref{prop:termination}. Consider the policy that plays action $a$ at $s$ and thereafter maximizes the expected time to termination. Since $T_s$ is the supremum over policies,
\[
    T_s \;\geq\; 1 + \sum_{s' \in \S}\P(s' \mid \SApair{})\, T_{s'},
    \qquad s \in \Snt,
\]
with $T_{s'} = 0$ for $s' \in \Sterm$. Hence for every $a \in \A(s)$,
\begin{equation}\label{eq:Tbellman-ineq}
    \sum_{s' \in \Snt} \P(s' \mid \SApair{})\, T_{s'}
    = \sum_{s' \in \S} \P(s' \mid \SApair{})\, T_{s'}
    \le T_s - 1.
\end{equation}
Since $\max_{a' \in \A(s')} w(s', a') = T_{s'}$, applying~\eqref{eq:Tbellman-ineq},
\[
\begin{aligned}
    \beta
    &= \maxSAnt \frac{\sum_{s' \in \Snt} \P(s' \mid \SApair{})\, T_{s'}}{T_s} \\
    &\le \maxSAnt \frac{T_s - 1}{T_s}
    = 1 - \frac{1}{\max_{s \in \Snt} T_s} \in [0,1),
\end{aligned}
\]
using $1 \le \max_s T_s < \infty$. The degenerate case $\beta = 0$ only strengthens the contraction.
\end{proof}

\begin{proof}[Proof of Lemma~\ref{lem:tsitsiklis}]
Taking the conditional expectation of the square of the recursion, the cross term vanishes since $\E[\err_t \mid \SApair{t}] = 0$:
\[
\begin{aligned}
    \E[W_{t+1}^2 \mid \SApair{t}]
    &= (1-\alpha_t)^2 W_t^2 + \alpha_t^2\, \E[\err_t^2 \mid \SApair{t}] \\
    &\le (1-\alpha_t)^2 W_t^2 + \alpha_t^2 B.
\end{aligned}
\]
Taking total expectations and iterating (via the tower property),
\[
    \E[W_{t+1}^2] \le \prod_{k=0}^t (1-\alpha_k)^2\, \E[W_0^2]
    + B \sum_{k=0}^t \alpha_k^2 \prod_{j=k+1}^t (1-\alpha_j)^2.
\]
As $t \to \infty$ the first term vanishes because $\sum_t \alpha_t = \infty$ forces $\prod_t (1-\alpha_t) \to 0$. For the second term, splitting the sum at $N$ and bounding the tail products by $1$,
\[
    B \sum_{k=0}^\infty \alpha_k^2 \prod_{j>k} (1-\alpha_j)^2
    \le B \sum_{k=N+1}^\infty \alpha_k^2,
\]
which is below any $\zeta > 0$ for $N$ large, since $\sum_t \alpha_t^2 < \infty$. Hence $\lim_{t\to\infty}\E[W_{t+1}^2] = 0$. Separately, $\E[W_{t+1}^2 \mid \SApair{t}] \le W_t^2 + \alpha_t^2 B$ with $\sum_t \alpha_t^2 < \infty$, so by the Robbins--Siegmund theorem $W_t^2$ converges almost surely to a limit $W_\infty^2$. Fatou's lemma then gives $\E[W_\infty^2] \le \liminf_t \E[W_t^2] = 0$, i.e., $W_\infty = 0$ almost surely.
\end{proof}

\begin{proof}[Proof of Theorem~\ref{thm:main}]
By the decomposition~\eqref{eq:sa-update}, subtracting $Q^*$ and writing $\Delta_t(\SApair{}) = Q_t(\SApair{}) - Q^*(\SApair{})$ for the error gives
\begin{equation}\label{eq:error_recursion}
\begin{split}
    \Delta_{t+1}(\SApair{}) ={}& (1 - \alpha_t) \Delta_t(\SApair{}) \\
    &+ \alpha_t [ (\H Q_t)(\SApair{}) - Q^*(\SApair{})] \\
    &+ \alpha_t \err_t(\SApair{}).
\end{split}
\end{equation}
We show $\wnorm{\Delta_t} \to 0$. Fix $\delta \in (0, (1-\beta)/2)$ so that $\beta + 2\delta < 1$ (possible since $\beta < 1$), and set $X_0 = 1$, $X_k = (\beta + 2\delta) X_{k-1}$. We prove by induction on $k$ that there is a random time $t_k$ with $\wnorm{\Delta_t} \le X_k$ for all $t \ge t_k$.

For $t \ge t_k$, define a deterministic process $Y_t$ and a noise process $W_t$ per pair by $Y_{t_k} = X_k$, $W_{t_k} = 0$, and
\[
\begin{aligned}
    Y_{t+1}(\SApair{}) &= (1-\alpha_t)\,Y_t(\SApair{}) + \alpha_t\,\beta X_k, \\
    W_{t+1}(\SApair{}) &= (1-\alpha_t)\,W_t(\SApair{}) + \alpha_t\,\frac{\err_t(\SApair{})}{w(\SApair{})}.
\end{aligned}
\]
Since $|\err_t| \le 1$ and $w = T_s \ge 1$, the normalized noise $\err_t/w$ has zero conditional mean and second moment at most $1$. By Lemma~\ref{lem:tsitsiklis}, applied to each of the finitely many pairs, $\max_{\SAnt} |W_t| \to 0$ almost surely. Also $Z_t \coloneqq Y_t - \beta X_k$ satisfies $Z_{t+1} = (1-\alpha_t) Z_t$, so $Z_t \to 0$ deterministically (as $\sum_t \alpha_t = \infty$), and thus $Y_t \to \beta X_k$.

\emph{Base case ($k=0$).} $|\Delta_t| \le X_0 = 1$ for all $t$, so $t_0 = 0$.

\emph{Inductive step.} Assume $\wnorm{\Delta_t} \le X_k$ for all $t \ge t_k$. We show, by induction on $t \ge t_k$, that
\begin{equation}\label{eq:main_induction_target}
    |\Delta_t(\SApair{}) - w(\SApair{})\,W_t(\SApair{})| \le w(\SApair{})\,Y_t(\SApair{})
\end{equation}
for all $(\SApair{}) \in \SAnt$. Equivalently, $w(W_t - Y_t) \le \Delta_t \le w(W_t + Y_t)$. It holds at $t_k$ since $W_{t_k} = 0$, $Y_{t_k} = X_k$, and $\wnorm{\Delta_{t_k}} \le X_k$. For the step, using~\eqref{eq:error_recursion}, the inductive hypothesis, and the contraction bound $|\H Q_t - Q^*| \le w\,\beta X_k$ (Property~\ref{prop:contraction} with $\wnorm{\Delta_t} \le X_k$),
\[
\begin{aligned}
    \Delta_{t+1}
    &\le (1-\alpha_t)\,w(Y_t + W_t) + \alpha_t w \beta X_k + \alpha_t \err_t \\
    &= w\,Y_{t+1} + w\,W_{t+1},
\end{aligned}
\]
and symmetrically $\Delta_{t+1} \ge -w\,Y_{t+1} + w\,W_{t+1}$, giving~\eqref{eq:main_induction_target} at $t+1$. Dividing by $w$ and maximizing over $\SAnt$,
\[
\begin{aligned}
    \wnorm{\Delta_{t+1}} \le{}& \maxSAnt |Y_{t+1}(\SApair{})| \\
    &+ \maxSAnt |W_{t+1}(\SApair{})|.
\end{aligned}
\]
Since $Y_t \to \beta X_k$ and $W_t \to 0$, there is a time $t_{k+1}$ beyond which $\maxSAnt Y_t \le (\beta + \delta) X_k$ and $\maxSAnt |W_t| \le \delta X_k$, hence $\wnorm{\Delta_t} \le (\beta + 2\delta) X_k = X_{k+1}$.

Thus $t_k$ exists for every $k$, and since $X_k = (\beta+2\delta)^k \to 0$, we conclude $\wnorm{Q_t - Q^*} = \wnorm{\Delta_t} \to 0$ almost surely, so $Q_t(\SApair{}) \to Q^*(\SApair{})$ for every reachable $(\SApair{})$.
\end{proof}

\begin{proof}[Proof of Corollary~\ref{cor:asymptotic}]
Two steps. First, greediness eventually selects only $\Qstar$-optimal actions. Call $a \in \A(s)$ \emph{strictly suboptimal} at $s$ if $\Qstar(\SApair{}) < \Vstar(s)$, and let $\kappa = \min\{\Vstar(s) - \Qstar(\SApair{})\}$ over all such pairs, a minimum over a finite set of positive numbers and hence positive (if there are none, every greedy policy is already $\Qstar$-greedy and this step is vacuous). On the almost sure event of Theorem~\ref{thm:main} there is a time after which $|Q_t - \Qstar| < \kappa/2$ at every pair, and then no strictly suboptimal action can attain $\max_a Q_t(s,a)$, so $\pi_t$ is $\Qstar$-greedy from that time on. Ties among \emph{optimal} actions are harmless, since the argument constrains only which actions are excluded.

Second, every $\Qstar$-greedy policy is optimal: such a $\pi$ satisfies~\eqref{def:vboe} with the maximum realized, so $V^\pi$ and $\Vstar$ agree as fixed points, uniqueness again coming from Assumption~\ref{asm:finite_hitting} via Property~\ref{prop:contraction} specialized to $\pi(s)$. The assumption is not removable: under an end component a $\Qstar$-greedy policy may cycle forever and never reach $\ST$, giving $V^\pi = 0 < \Vstar$. Combining, $\pi_t$ is optimal for all large $t$, so $V^{\pi_t}(\snot) = \Vstar(\snot)$ eventually. 
\end{proof}

\section{Backward-Replay Proof}
\label{app:br}

Section~\ref{sec:algo} introduced \QUASARBR{}, which buffers an episode and applies the identical Line-12 updates in reverse chronological order once the episode ends. This sup.\ mat.\ proves that the guarantee of Theorem~\ref{thm:main} extends to this variant. It is worth being precise about why a separate argument is needed. The noise term of~\eqref{eq:sa-update} is defined against the table the update reads, and the proof of Theorem~\ref{thm:main} rests on that table being determined by the past of the sampled successor, so that $\E[\err_t \mid \SApair{t}] = 0$. Under backward replay this fails: when the $k$-th transition of an episode is replayed, the table already contains the replayed updates of the chronologically \emph{later} transitions, so the target consulted at time $k$ is a function of draws made after $s_{k+1}$. The conditioning that pins down the table also pins down the sample it is supposed to predict, and no choice of filtration restores the martingale property of the raw noise.

Two elementary observations repair the argument. First, every replayed target can be re-anchored at the table as it stood at the episode boundary; the re-anchored noise is a martingale difference with respect to the ordinary sampling filtration, and the anchoring error---the movement of the bootstrap row during the flush---is bounded by the total step size $A_e$ the episode spends. Second, the reversal itself merely reorders a composition of affine maps, and reordering perturbs the composite only at \emph{second} order in the step sizes, while every legitimate term enters at first order. Both prices vanish, since $A_e \to 0$ almost surely, and the ceiling induction of Theorem~\ref{thm:main} then goes through stepped at episode boundaries, with the deterministic envelope inflated by a vanishing amount. We first fix the analyzed object and the notation, then prove the four lemmas the induction consumes, and finally run the induction.

\Theoremspace
\begin{definition}[\QUASARBR{}]\label{def:br}
Algorithm~\ref{alg:qlearning} with the updates deferred to the episode boundary: during an episode the table is not written; each transition is buffered, its pair's counter being incremented and the step size $\alpha_{N(\SApair{})}(\SApair{})$ \emph{recorded at the moment the transition is drawn}; when the episode ends---on reaching $\Sterm$, or truncated after $H \ge |\Snt|$ steps---the buffered Line-12 updates are applied in reverse chronological order, each with its recorded step size. Truncation triggers the flush like any other episode end and is not treated as termination: the final update bootstraps through $\max_{a'} Q(s_{L_e}, a')$ as usual.
\end{definition}

\paragraph{Notation.}
Episodes are indexed by $e$, with length $L_e \le H$, buffered transitions $(\SApair{k}, s_{k+1})$ for $k = 0, \dots, L_e - 1$, and recorded step sizes $\alpha_k$. Let $\bar Q_e$ denote the table read throughout episode $e$ (the table is written only at boundaries), and let $Q^{(k)}$ denote the table after the flush has replayed the transitions down to index $k$, so that $Q^{(L_e)} = \bar Q_e$ and $Q^{(0)} = \bar Q_{e+1}$. Write
\[
    A_e \coloneqq \sum_{k < L_e} \alpha_k \le H,
\]
for the episode's total step-size budget, and $A_e(\SApair{})$ for the part contributed by the visits to the pair $(\SApair{})$. The sampling filtration is $\mathcal{F}_{e,k} \coloneqq \sigma(\mathcal{G}_{e-1},\, s_0, a_0, \dots, s_k, a_k)$, where $\mathcal{G}_{e-1}$ collects everything through the end of episode $e-1$; re-indexing the pairs $(e,k)$ by the global sampling time $t$ recovers the filtration $\{\mathcal{F}_t\}$ of Lemma~\ref{lem:tsitsiklis}, with the off-pair convention $\alpha_t = 0$ of Section~\ref{sec:convergence} unchanged. Finally, $\tau_e$ denotes the sampling time at which the flush of episode $e$ has just been applied.

Adding and subtracting the boundary-anchored target, the replay of transition $k$ reads, exactly,
\begin{equation}\label{eq:br-decomp}
\begin{split}
    Q^{(k)}(\SApair{k}) ={}& (1 - \alpha_k)\, Q^{(k+1)}(\SApair{k}) \\
    &+ \alpha_k \bigl[ (\H \bar Q_e)(\SApair{k}) + \err_k + \xi_k \bigr],
\end{split}
\end{equation}
where
\[
\begin{aligned}
    \err_k &\coloneqq \max_{a' \in \A(s_{k+1})} \bar Q_e(s_{k+1}, a') - (\H \bar Q_e)(\SApair{k}), \\
    \xi_k &\coloneqq \max_{a'} Q^{(k+1)}(s_{k+1}, a') - \max_{a'} \bar Q_e(s_{k+1}, a').
\end{aligned}
\]
For terminal $s_{k+1}$ the maxima read the clamped rows of Line~3, so the expressions apply verbatim and $\xi_k = 0$ there. Nothing has been changed: \eqref{eq:br-decomp} is the algorithm's update, rewritten. The point of the rewriting is where each piece is measurable: $\err_k$ is built from the frozen table and the current transition only, while everything the reversed order touches is confined to $\xi_k$, which will be bounded pathwise and never conditioned on.

\Theoremspace
\begin{lemma}[Vanishing budget]\label{lem:br-budget}
Under Definition~\ref{def:br}, $A_e \to 0$ almost surely.
\end{lemma}

\begin{proof}
By Property~\ref{prop:exploration}---which applies verbatim here, its exploration floor being indifferent to which table the $\epsilon$-greedy policy reads and to the truncation---every pair's visit count tends to infinity, hence so does $M_e$, the least count over $\SAnt$ entering episode $e$. Assumption~\ref{asm:lr} forces $\alpha_n \to 0$, so $s(M) \coloneqq \sup_{n > M} \alpha_n \to 0$. Every step applied in episode $e$ is drawn at a count exceeding $M_e$, so $A_e \le H\, s(M_e) \to 0$.
\end{proof}

\Theoremspace
\begin{lemma}[Frozen noise]\label{lem:br-noise}
$\err_k$ is the noise of~\eqref{eq:sa-update} with $\bar Q_e$ in place of $Q_t$, and it retains its properties: $\E[\err_k \mid \mathcal{F}_{e,k}] = 0$, $\E[\err_k^2 \mid \mathcal{F}_{e,k}] \le 1$, the recorded step sizes are $\mathcal{F}_{e,k}$-measurable, and Assumption~\ref{asm:lr} holds in global time. In particular Lemma~\ref{lem:tsitsiklis} applies to any process driven by $\alpha_t \err_t$.
\end{lemma}

\begin{proof}
The argument of Section~\ref{sec:convergence} goes through verbatim with $\bar Q_e$ in the role of $Q_t$: the anchor is $\mathcal{G}_{e-1}$-measurable, so given $\mathcal{F}_{e,k}$ the sampled maximum averages over the true kernel to $(\H \bar Q_e)(\SApair{k})$, and both terms lie in $[0,1]$. Restoring this argument is the entire point of the re-anchoring; it is the mid-flush table that is measurable with respect to no admissible past. Step-size measurability is the recording convention of Definition~\ref{def:br}, and the multiset of (pair, count) assignments coincides with the online algorithm's, so the discussion following Assumption~\ref{asm:lr} applies unchanged.
\end{proof}

\Theoremspace
\begin{lemma}[Drift]\label{lem:br-drift}
Within the flush of episode $e$, every intermediate table satisfies $\|Q^{(j)} - \bar Q_e\|_\infty \le A_e$. Consequently $|\xi_k| \le A_e$ for every $k$, and $\|\bar Q_{e+1} - \bar Q_e\|_\infty \le A_e$.
\end{lemma}

\begin{proof}
Each replayed update changes exactly one entry, by $\alpha_k\,|\text{target} - \text{current}| \le \alpha_k$, since every table entry and every target lies in $[0,1]$; telescoping over the replayed transitions bounds each intermediate deviation by $A_e$. The maximum is non-expansive, $|\max_{a'} f - \max_{a'} g| \le \max_{a'} |f - g|$, so $\xi_k$ inherits the bound.
\end{proof}

\Theoremspace
\begin{lemma}[Affine reordering]\label{lem:br-reorder}
Let $\varphi_i(v) = (1-\alpha_i)v + \alpha_i g_i$ for $i = 1, \dots, m$, with $\alpha_i \in [0,1]$, and let $A' \coloneqq \sum_i \alpha_i$. Applied in any order, the composite evaluates at $v_0$ to
\[
    \Pi\, v_0 + \sum_{i=1}^m c_i\, g_i,
    \qquad
    \Pi \coloneqq \prod_{j=1}^m (1 - \alpha_j),
\]
where $c_i = \alpha_i \prod_j (1-\alpha_j)$, the product over the maps applied \emph{after} $\varphi_i$. Moreover $\alpha_i (1 - A') \le c_i \le \alpha_i$ for every order; hence the coefficients of two orders differ by $|c_i - c_i'| \le \alpha_i A'$.
\end{lemma}

\begin{proof}
Induction on $m$: slopes multiply, and each injection $\alpha_i g_i$ is scaled by the slopes of the maps applied after it. For the bounds, $c_i / \alpha_i$ is a product of factors $1 - \alpha_j$ over a subset of $\{j \ne i\}$, and $\prod_j (1-\alpha_j) \ge 1 - \sum_j \alpha_j$; two numbers in $[1 - A', 1]$ differ by at most $A'$.
\end{proof}

\Theoremspace
\begin{lemma}[Vanishing input]\label{lem:br-vanish}
Let $Z_{t+1} = (1-a_t)Z_t + a_t c_t$ pathwise, with $a_t \in [0,1]$, $\sum_t a_t = \infty$, $Z_0 \ge 0$, $c_t \ge 0$, and $c_t \to 0$. Then $Z_t \to 0$.
\end{lemma}

\begin{proof}
Given $\rho > 0$, pick $T$ with $c_t \le \rho$ for $t \ge T$. Then $(Z_{t+1} - \rho)^+ \le (1-a_t)(Z_t - \rho)^+$, so $(Z_t - \rho)^+ \le \prod_{r=T}^{t-1}(1-a_r)\,(Z_T - \rho)^+ \to 0$, since $\sum_t a_t = \infty$. Hence $\limsup_t Z_t \le \rho$; let $\rho \downarrow 0$ and use $Z_t \ge 0$.
\end{proof}

\Theoremspace
\begin{theorem}\label{thm:br}
Under Assumptions~\ref{asm:finite_hitting} and~\ref{asm:lr}, the \QUASARBR{} iterates of Definition~\ref{def:br} satisfy $Q_t(\SApair{}) \to Q^*(\SApair{})$ almost surely for every $(\SApair{}) \in \SAnt$.
\end{theorem}

Corollary~\ref{cor:asymptotic} holds verbatim with Theorem~\ref{thm:br} in place of Theorem~\ref{thm:main}, since its proof consumes value convergence as a black box.

\begin{proof}
The table is written only at episode boundaries, so the table read at every sampling time is $\bar Q_{e(t)}$, and it suffices to control the boundary errors $\Delta_e \coloneqq \bar Q_e - Q^*$; mid-flush tables are then within $A_e \to 0$ of $\bar Q_e$ in the supremum norm by Lemmas~\ref{lem:br-drift} and~\ref{lem:br-budget}. As in the proof of Theorem~\ref{thm:main}, fix $\delta \in (0, (1-\beta)/2)$ and set $X_0 = 1$, $X_k = (\beta + 2\delta)X_{k-1}$. We prove by induction on $k$ that there is a random episode index $e_k$ with $\wnorm{\Delta_e} \le X_k$ for all $e \ge e_k$; since $X_k \to 0$, this proves the theorem. Every $e_k$ is taken beyond the last episode containing a first visit (there are at most $|\SAnt|$ such episodes, since each consumes at least one first visit) and beyond the last episode with $A_e > \tfrac12$ (almost surely finite, by Lemma~\ref{lem:br-budget}); the base case $k = 0$ holds from that point on since every table lies in $[0,1]^{\Snt \times \A}$ and $w \ge 1$.

For the inductive step, assume $\wnorm{\Delta_e} \le X_k$ for all $e \ge e_k$. Define per pair, on the global sampling clock from $\tau_{e_k}$, with $W_{\tau_{e_k}} = 0$, $Y_{\tau_{e_k}} = X_k$, and $c_e \coloneqq 6 A_e$:
\[
\begin{aligned}
    W_{t+1}(\SApair{}) &= (1-\alpha_t)\,W_t(\SApair{}) + \alpha_t\,\frac{\err_t(\SApair{})}{w(\SApair{})}, \\
    Y_{t+1}(\SApair{}) &= (1-\alpha_t)\,Y_t(\SApair{}) + \alpha_t\,\bigl(\beta X_k + c_{e(t)}\bigr).
\end{aligned}
\]
By Lemmas~\ref{lem:br-noise} and~\ref{lem:tsitsiklis}, $\maxSAnt |W_t| \to 0$ almost surely. The envelope's target is inflated by $c_e$, so $Z_t \coloneqq Y_t - \beta X_k$ satisfies $Z_{t+1} = (1-\alpha_t)Z_t + \alpha_t c_{e(t)}$, no longer the exact geometric decay of Theorem~\ref{thm:main}; since $c_e \to 0$ (Lemma~\ref{lem:br-budget}), Lemma~\ref{lem:br-vanish} still gives $Y_t \to \beta X_k$.

We show, by induction on episode boundaries $e \ge e_k$, the sandwich
\begin{equation}\label{eq:br-sandwich}
    |\Delta_e(\SApair{}) - w(\SApair{})\,W_{\tau_e}(\SApair{})|
    \le w(\SApair{})\,Y_{\tau_e}(\SApair{})
\end{equation}
at every pair. It holds at $e_k$ since $W = 0$, $Y = X_k$, and $\wnorm{\Delta_{e_k}} \le X_k$. For the step $e \to e+1$, fix a pair $x$ visited $m \ge 1$ times in episode $e$ (unvisited pairs are idle on both sides), with recorded steps $\alpha_1, \dots, \alpha_m$ in sampling order and $A' \coloneqq \sum_i \alpha_i = A_e(x) \le A_e$. Each update writes only its own pair's entry, so by~\eqref{eq:br-decomp} the flush acts on the entry at $x$ by composing $x$'s own maps, with targets $b + \err_i + \xi_i$ where $b \coloneqq (\H \bar Q_e - Q^*)(x)$, in reverse order; everything the other pairs' updates do enters only through the $\xi_i$, already bounded. Unrolling the flush and the $W, Y$ recursions across the episode (Lemma~\ref{lem:br-reorder}, with the two coefficient systems identified above),
\[
\begin{aligned}
    \Delta_{e+1}(x) ={}& \Pi\,\Delta_e(x)
    + \sum_i \alpha_i P_i^{<} \bigl[ b + \err_i + \xi_i \bigr], \\
    w(W{+}Y)_{\tau_{e+1}}(x) ={}& \Pi\, w(W{+}Y)_{\tau_e}(x) \\
    &+ \sum_i \alpha_i P_i^{>} \bigl[ w(\beta X_k + c_e) + \err_i \bigr].
\end{aligned}
\]
Subtracting, the inherited terms compare by the sandwich at $e$ (their common coefficient $\Pi \ge 0$ is order-invariant), and $|b| \le w \beta X_k$ by Property~\ref{prop:contraction} with the outer hypothesis at episode $e$---legitimate precisely because the anchor $\bar Q_e$ is a \emph{boundary} table, covered by that hypothesis. Hence
\[
\begin{split}
    &\bigl[w(W{+}Y) - \Delta\bigr]_{e+1}(x) \\
    &\;\;\ge\; \sum_i \alpha_i \bigl(P_i^{>} - P_i^{<}\bigr)\bigl(w \beta X_k + \err_i\bigr) \\
    &\;\;\quad + w\, c_e \sum_i \alpha_i P_i^{>} \;-\; \sum_i \alpha_i P_i^{<}\, \xi_i \\
    &\;\;\ge\; -2 w A'^2  + 3 w A' A_e - w A' A_e \\
    &\;\;=\; 2 w A' (A_e - A') \;\ge\; 0,
\end{split}
\]
using, in order: $|P_i^{>} - P_i^{<}| \le A'$ (Lemma~\ref{lem:br-reorder}) with $|w \beta X_k + \err_i| \le 2w$ ($w \geq 1 \geq |n_i|$); the lower bound $\sum_i \alpha_i P_i^{>} \ge A'(1 - A') \ge A'/2$ together with $c_e = 6A_e$; and $|\xi_i| \le A_e$ (Lemma~\ref{lem:br-drift}). The symmetric computation bounds $[\Delta - w(W{-}Y)]_{e+1}(x)$ below by the same quantity, giving~\eqref{eq:br-sandwich} at $e+1$. 

Dividing by $w$ and maximizing over $\SAnt$, $\wnorm{\Delta_e} \le \maxSAnt Y_{\tau_e} + \maxSAnt |W_{\tau_e}|$; since $Y_t \to \beta X_k$ and $W_t \to 0$, there is an episode index $e_{k+1}$ beyond which $\wnorm{\Delta_e} \le (\beta + \delta)X_k + \delta X_k = X_{k+1}$. The induction advances, and $X_k \to 0$ completes the proof.
\end{proof}

Backward replay is moreover exact after a single episode on the instances whose optimum is $1$, which accounts for its single-episode convergence there.
\Theoremspace
\begin{theorem}[One-pass exactness]\label{thm:onepass}
Suppose $\Qstar(\SApair{}) = 1$ for every pair reachable from $\snot$, equivalently every reachable terminal is a target. Then \QUASARBR{} with first-visit step size $\alpha = 1$ sets $Q(\SApair{}) = \Qstar(\SApair{})$ on every pair of the first target-reaching trajectory after a single backward sweep, whatever the length of that trajectory.
\end{theorem}
\begin{proof}
Let $\tau = (s_0, a_0, \dots, s_H)$ be the first target-reaching episode, so $s_H \in \Sterm$. By hypothesis every terminal reachable from $\snot$ is a target, so $s_H \in \ST$ with boundary value $1$, and $\Qstar(\SApair{}) = 1$ at every reachable pair. Backward replay updates the pairs $(s_k, a_k)$ in the order $k = H-1, \dots, 0$, and on a first visit the step size is $\alpha = 1$, so the update sets $Q(s_k, a_k) = \max_{a' \in \A(s_{k+1})} Q(s_{k+1}, a')$. Assume first that $\tau$ visits each pair at most once. We show by downward induction that $Q(s_k, a_k) = 1$ after its update. For $k = H-1$ the successor $s_H$ is a target, so the maximum is $1$. For $k < H-1$ the pair $(s_{k+1}, a_{k+1})$ was updated earlier in the sweep, to $1$ by the inductive hypothesis, so $\max_{a'} Q(s_{k+1}, a') \ge 1$, and since every entry lies in $[0,1]$ the update sets $Q(s_k, a_k) = 1$. As $\Qstar \equiv 1$ on these pairs, the assignment is exact. If $\tau$ repeats a pair, apply the argument to its earliest-replayed (chronologically last) occurrence. Because the maximum is monotone and values are capped at $1$, the remaining occurrences leave the entry at $1$.
\end{proof}

\section{Reachability versus Stochastic Shortest Path}
\label{app:ssp}

We give two small examples that separate reachability from the stochastic shortest path (SSP) problem. The first is a reachability instance that is not an SSP instance at all. The second shows that even where both are defined, their optima disagree.

Consider the MDP with a single non-terminal state $\snot$ and two terminal states, a target $g \in \ST$ and a non-target sink $d \in \SN$. Two actions are enabled at $\snot$ (Figure~\ref{fig:ssp-example}). Action $a$ moves to $g$ with probability $\tfrac{1}{2}$ and to $d$ with probability $\tfrac{1}{2}$, and action $b$ moves to $g$ with probability $\tfrac{3}{4}$ and to $d$ with probability $\tfrac{1}{4}$. The reachability value is $\Vstar(\snot) = \max\{\tfrac{1}{2}, \tfrac{3}{4}\} = \tfrac{3}{4}$, attained by $b$, and the no-EC assumption holds since every policy terminates in one step. This is not an SSP instance. SSP is defined only when some policy is \emph{proper}, that is, reaches the target with probability one, whereas here every policy sends positive mass to the sink $d$ from which $g$ is unreachable. Reachability remains well posed precisely because it asks for a probability rather than a path, and that probability $\tfrac{3}{4}$ is the content the SSP formulation cannot express.

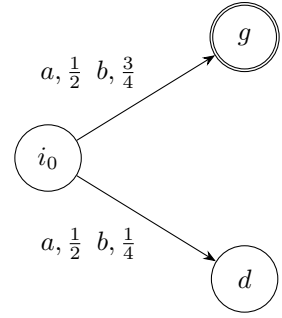
\begin{figure}[t]
\centering

\begin{tikzpicture}[->, >=Stealth, node distance=2.6cm, on grid]
  \node[state] (s0) {$\snot$};
  \node[state, accepting, above right=1.6cm and 2.6cm of s0] (g) {$g$};
  \node[state, below right=1.6cm and 2.6cm of s0] (d) {$d$};
  \path
    (s0) edge node[above left] {$a,\tfrac{1}{2}\ \ b,\tfrac{3}{4}$} (g)
    (s0) edge node[below left] {$a,\tfrac{1}{2}\ \ b,\tfrac{1}{4}$} (d);
\end{tikzpicture}

\caption{ A reachability instance that is not an SSP instance. From $\snot$, actions $a$ and $b$ each split their mass between the target $g$ and the sink $d$. No policy reaches $g$ with probability one, so no proper policy exists, yet the reachability value $\tfrac{3}{4}$ is well defined and attained by $b$.}
\label{fig:ssp-example}
\end{figure}

The two problems part ways even when every policy is proper. Suppose that from $\snot$ one action reaches $g$ in a single step with probability one, while a second action reaches $g$ with probability one along a chain of $k$ intermediate non-terminal states. Both actions have reachability value $1$, so both are optimal for reachability, which records only whether $g$ is reached and not how long the trajectory takes to reach it. SSP with unit step costs instead charges the direct action a cost of $1$ and the chain a cost of $k+1$, and strictly prefers the former. Reachability is thus blind to exactly the path length that SSP is designed to minimize, so the two objectives select different policies on one and the same MDP.